\documentclass[journal]{IEEEtran}
\usepackage{pst-all}
\usepackage{amsmath}
\usepackage{amssymb}
\usepackage{hyperref}
\usepackage{subfigure}
\usepackage{amsthm}
\usepackage{bm}
\usepackage{graphicx,epsfig,rotating} 
\usepackage{psfrag,texdraw}
\newcommand{\herm}[1]{{#1}^H} 
\newcommand{\transp}[1]{{#1}^T} 

\newtheorem{theorem}{Theorem}[section]

\newtheorem{lemma}[theorem]{Lemma}

\newtheorem{corollary}[theorem]{Corollary}
\newtheorem{algorithm}{Algorithm}
\newtheorem{assumption}{Assumption}

\def \expect {{\rm E} }
\def \prob {{\rm P} }
\def \twiddle[#1] {e^{-j \frac{2 \pi}{N}  #1 }}
\def \twiddleneg[#1] {e^{j \frac{2 \pi}{N}  #1 }}

\newcommand\defeq{:=}

\newcommand{\vbed}{\bm \beta} 

\newcommand\norm[2][\Tnorm]{\ensuremath{{\|#2\|}_{#1}}}

\newcommand{\vep}{\bm \varepsilon}
\newcommand{\vbe}{\bm \beta}

\newcommand{\lagvar}{m}
\newcommand{\freqbin}{f}
\newcommand{\samplesize}{N}
\newcommand{\nodes}{\mathcal{V}}
\newcommand{\edges}{\mathcal{E}}

\newcommand\vect[1]{\mathbf #1}

\newcommand{\vx}{\vect{x}}  
\newcommand{\vy}{\vect{y}}

\newcommand{\mG}{\vect{G}}

\newcommand{\mX}{\vect{X}}

\newcommand{\nrtasks}{F}

\newcommand{\coefflen}{p}
\newcommand{\measlen}{p}
\newcommand{\task}{f}
\newcommand{\sparsity}{s_{\text{max}}}

\renewcommand{\S}{\mathcal S}
\DeclareMathOperator{\gsupp}{gsupp}

\DeclareMathOperator*{\rank}{rank}

\DeclareMathOperator*{\argmin}{argmin}

\DeclareMathOperator*{\diag}{diag}
\DeclareMathOperator*{\EX}{E}
\DeclareMathOperator*{\PR}{P}

\def\ML_est{\hat{\mathbf{x}}_{\text{ML}}}

\newcommand{\ACF}{\mathbf{R}_{x}} 
\newcommand{\cig}{\mathcal{G}_{x}}

\newcommand{\SDM}{\mathbf{S}_{x}}  
\newcommand{\ESDM}{\widehat{\mathbf{S}}_{x}}  
\newcommand{\EACF}{\widehat{\mathbf{R}}_{x}}

\newcommand{\be}{\begin{equation}}
\newcommand{\ee}{\end{equation}}

\newcommand{\maxdegree}{s_{\text{max}}}

 \newcommand{\mc}{\mathcal}

\newcommand\comp[1]{ {#1}^c}
\newcommand{\autocovfunc}{\mathbf{R}_{x}} 
\newcommand{\timeidx}{n}

\ifCLASSINFOpdf
\else
\fi

\begin{document}
%
\title{{Learning the Conditional Independence Structure of Stationary Time Series:\\ A Multitask Learning Approach}}
%
%
%

\author{Alexander~Jung~\IEEEmembership{}
 \thanks{A. Jung is with the Institute of Telecommunications, Vienna University of Technology, 1040-Vienna,
Austria e-mail: ajung@nt.tuwien.ac.at.}
\thanks{Manuscript revised \today}}
       
\maketitle

\begin{abstract}
We propose a method for inferring the conditional independence graph (CIG) of a high-dimensional Gaussian vector time series (discrete-time process) from 
a finite-length observation. By contrast to existing approaches, we do not rely on a parametric process model (such as, e.g., an autoregressive model) for the observed random process. Instead, we only require certain smoothness properties (in the Fourier domain) of the process. The proposed inference scheme {works even for sample sizes much smaller than the number of scalar process components if the true underlying CIG is sufficiently sparse.}
A theoretical performance analysis provides conditions which guarantee that the probability of the proposed inference method to deliver a wrong CIG is below a prescribed value. These conditions imply lower bounds on the sample size such that the new method is consistent asymptotically. Some numerical experiments validate our theoretical performance analysis and demonstrate superior performance of our scheme compared to an existing (parametric) approach in case of model mismatch.
\end{abstract}


\IEEEpeerreviewmaketitle

\section{{Introduction}}\label{sec.intro}
\IEEEPARstart{W}{e} consider a stationary discrete-time vector process or time series. 
Such a process could model, e.g., the time evolution of air pollutant concentrations \cite{Dahlhaus2000,DahlhausEichler2003} 
or medical diagnostic data obtained in electrocorticography (ECoG) \cite{Nowak2011}. 

One specific way of representing the dependence structure of a vector process is via a graphical model \cite{LauritzenGM}, where the nodes of the graph represent the individual scalar process components, and the edges represent statistical relations between the individual process components. More precisely, the (undirected) edges of a \emph{conditional independence graph (CIG)} associated with a process represent conditional independence statements about the process components \cite{LauritzenGM,Dahlhaus2000}. In particular, two nodes in the CIG are connected by an edge if and only if the two corresponding process components are conditionally dependent, given the remaining process components. 
{Note that the so defined CIG for time series extends the basic notion of a CIG for random vectors by considering dependencies between entire time series instead of 
dependencies between scalar random variables \cite{BachJordan04,Brillinger96remarksconcerning}.}

In this work, we investigate the problem of graphical model selection (GMS), i.e., that of inferring the CIG of a time series, given a finite-length observation. 

Our work applies to the \emph{high-dimensional} regime, where the model dimension, given by the number of process components, is allowed to be (much) larger than the amount of observed data, given by the sample size \cite{ElKaroui08,Santhanam2012,RavWainLaff2010,Nowak2011,Bento2010,MeinBuhl2006,FriedHastieTibsh2008}. 
It is then intuitively clear that some additional problem structure is required in order to allow for the existence of consistent estimation schemes. Here, this structure is given by sparsity constraints placed on the CIG. More precisely, we assume that the underlying CIG has a small maximum node degree{, i.e., each node has a relatively small number of neighbors.}  

\paragraph{Existing Work}

GMS for high-dimensional processes with observations modeled as i.i.d.\ is now well developed \cite{Zhao2006,MeinBuhl2006,RavWainLaff2010}.
For continuous valued Gaussian Markov random fields, binary Ising models as well as mixed graphical models (containing both continuous and discrete random variables), efficient neighborhood regression based approaches to infer the underlying graphical model have been proposed \cite{MeinBuhl2006,RavWainLaff2010,LeeHastieMixedModels}. An alternative to the local neighborhood regression approach is based on the minimization of a $\ell_{1}$-norm penalized log-likelihood function \cite{RavWainRaskYu2011}. The authors of \cite{MeinBuhl2006,RavWainLaff2010,ElKaroui08,RavWainRaskYu2011} present sufficient conditions such that their proposed recovery method is consistent in the high-dimensional regime. 
These sufficient conditions are complemented by the 
fundamental performance limits derived in \cite{WangWain2010}, showing that in certain regimes the (computationally efficient) selection scheme put forward in \cite{RavWainRaskYu2011} performs essentially optimal. 
 
 {The common feature of existing approaches to GMS for temporally correlated vector processes is that they are based on finite dimensional parametric 
 process models. Some of these approaches apply the recent theory of compressed sensing (CS) to learning dependence networks of vector processes using a vector 
 autoregressive process model \cite{Nowak2011,Bento2010,Songsiri09,songsiri2010}.}
 

\paragraph{Contribution} 

In this paper, we develop and analyze a \emph{nonparametric compressive GMS scheme} for general stationary time series. Thus, by contrast to existing approaches, we do not rely on a specific finite dimensional model for the observed process. Instead, we require the observed process to be sufficiently smooth in the spectral domain. This smoothness constraint will be quantified by certain moments of the process autocovariance function (ACF) and requires the ACF to be effectively supported on a small interval, whose size is known beforehand, e.g., due to specific domain knowledge.

Inspired by a recently introduced neighborhood regression based GMS method \cite{MeinBuhl2006} for Gaussian Markov random fields using i.i.d. samples, we propose a GMS method for time series by generalizing the neighborhood regression approach {for GMS} to the Fourier domain. Our approach combines an established method for  nonparametric spectrum estimation with a CS recovery method. 
The resulting method exploits a specific problem structure, inherent to the GMS problem, which corresponds to a special case of a \emph{block-sparse recovery problem} \cite{BlockSparsityEldarTSP,MishaliEldar2008,EldarRauhut2010}, i.e., a \emph{multitask learning problem} \cite{BuhlGeerBook,Lounici09}.  

Our main conceptual contribution is the formulation of GMS for time series as a multitask learning problem. Based on this 
formulation, we develop a GMS scheme by combining a Blackman-Tukey (BT) spectrum estimator with the \emph{multitask LASSO (mLASSO)} \cite{BuhlGeerBook,Lee_adaptivemulti-task}. {The distinctive feature of the multitask learning problem obtained here is that it is defined over a continuum of tasks, which are indexed 
by a frequency variable $\theta \in [0,1)$.} 
We also carry out a theoretical performance analysis of our selection scheme, by upper bounding the probability of our scheme to deliver a wrong CIG. 
Moreover, we assess the empirical performance of the proposed scheme by means of illustrative numerical experiments. 

\paragraph{Outline of the Paper}
We formalize the problem of GMS for stationary time series in Section \ref{SecProblemFormulation}. 
Our novel compressive GMS scheme for stationary processes is presented in Section \ref{sec_novel_sel_scheme}, which is organized in two parts. 
First, we discuss the spectrum estimator employed in our selection scheme. 
Then, we show how to apply the mLASSO for inferring the CIG, by 
formulating GMS for time series as a multitask learning problem. 
In Section \ref{sec_var_sel_consist}, we present a theoretical performance guarantee in the form of an upper bound on the probability of our algorithm to fail in correctly recovering the true underlying CIG. 
Finally, the results of some illustrative numerical experiments are presented in Section \ref{sec_numerical_experiments}.

\emph{Notation and basic definitions}.\, 
The modulus, complex conjugate, and real part of a complex number $a \! \in \! \mathbb{C}$ are denoted by $|a|$, $a^{*}$, $\Re\{a\}$,
respectively. The imaginary unit is denoted as $j \defeq \sqrt{-1}$. 
Boldface lowercase letters denote column vectors,  
whereas boldface uppercase letters denote matrices. 
The $k$th entry of a vector $\mathbf{a}$ is denoted by ${( \mathbf{a} )}_{k}$, and  the entry
of a matrix $\mathbf{A}$ in the $m$-th row and $n$-th column by ${( \mathbf{A} )}_{m,n}$. 
{The submatrix of $\mathbf{A}$ comprised of the elements in rows $a,\ldots,b$ and columns $c,\ldots,d$ is denoted $\mathbf{A}_{a:b,c:d}$.}
The superscripts $^{T}$, $^{*}$, and $^{H}$ denote the transpose, (entry-wise) conjugate, and Hermitian transpose, respectively. The $k$th column of the identity matrix will be denoted by $\mathbf{e}_{k}$.  

We denote by $\ell_{q}([0,1))$ the set of all vector-valued functions $\mathbf{c}(\cdot): [0,1) \rightarrow \mathbb{C}^{q}$ such that each component $c_{r}(\theta)$ is square integrable, i.e., $c_{r}(\cdot) \in L^{2}([0,1))$ (we also use the shorthand $L^{2}$) with norm $\|c_{r}(\cdot) \|_{L^{2}} \defeq \sqrt{\int_{\theta=0}^{1}  | c_{r}(\theta)|^{2} d \theta}$.
We then define the generalized support of $\mathbf{c}(\cdot) \in \ell_{q}([0,1))$ as $\gsupp(\mathbf{c}(\cdot)) \defeq \{ r \in [p] | \|c_{r}(\cdot)\|_{L^{2}} > 0 \}$. 
For $\mathbf{c}(\cdot) \in \ell_{q}([0,1))$ and a subset $\mathcal{I} \subseteq [q]$, we denote by $\mathbf{c}_{\mathcal{I}}(\cdot)$ the vector-valued function which is obtained by retaining 
only those components $c_{r}(\cdot)$ with $r \in \mathcal{I}$. 
Given $\mathbf{c}(\cdot) \in \ell_{q}([0,1))$, we define the norms $\| \mathbf{c} (\cdot) \|_{2} \defeq \sqrt{ \sum_{r \in [q]} \| c_{r}(\cdot) \|^{2}_{L^{2}} }$ and $\| \mathbf{c} (\cdot) \|_{1} \defeq \sum_{r \in [q]} \| c_{r}(\cdot) \|_{L^{2}}$, respectively.

Given a positive semidefinite (psd) matrix $\mathbf{C} \in \mathbb{C}^{\coefflen \times \coefflen}$, with eigenvalue decomposition $\mathbf{C} = \mathbf{U} \mathbf{D} \mathbf{U}^{H}$ with unitary $\mathbf{U} \in \mathbb{C}^{\coefflen \times \coefflen}$ and diagonal $\mathbf{D} \in \mathbb{R}_{+}^{\coefflen \times \coefflen}$, we denote its psd square root by $\sqrt{\mathbf{C}} \triangleq \mathbf{U} \sqrt{\mathbf{D}} \mathbf{U}^{H}$ where $\sqrt{\mathbf{D}}$ is defined entry-wise.

Given a matrix $\mathbf{H} \in \mathbb{C}^{\measlen \times \coefflen}$, we denote its spectral norm as $\| \mathbf{H} \|_{2} \defeq \max_{\mathbf{x} \neq \mathbf{0}} \frac{\| \mathbf{H} \mathbf{x} \|_{2}}{ \| \mathbf{x} \|_{2}}$. 
The norm $\| \mathbf{H} \|_{\infty}$ is defined as the largest magnitude of its entries, i.e., $\| \mathbf{H} \|_{\infty}\defeq \max\limits_{m,n} | (\mathbf{H})_{m,n}|$. 
\vspace*{-3mm}

\section{Problem Formulation}
\label{SecProblemFormulation}

Consider a $p$-dimensional stationary Gaussian random process $\mathbf{x}[\timeidx]$ with (matrix-valued) {ACF} $\ACF[\lagvar] \defeq \EX\{ \vx [\lagvar] \transp{\vx}[0] \}$, which is assumed 
to be summable, i.e., $\sum_{\lagvar = -\infty}^{\infty} \norm{\ACF[\lagvar]} < \infty$.\footnote{{The precise choice of norm is irrelevant for the definition of summability, since in finite-dimensional vector spaces all norms are equivalent \cite{MeyerMatrixAnalysis}.}} 

The \emph{spectral density matrix} (SDM) of the process $\mathbf{x}[\timeidx]$ is defined as
\vspace*{-2mm}
\begin{equation}
\label{equ_def_spectral_density_matrix}
\SDM(\theta) \defeq \sum_{\lagvar =-\infty}^{\infty} \ACF[\lagvar] \exp(-j2\pi \theta \lagvar). 
\vspace*{-2mm}
\end{equation}  
{
The SDM may be interpreted as the multivariate generalization of the power spectral density of a scalar stationary random process. In particular, by the 
multivariate spectral representation theorem, we can represent the vector process $\mathbf{x}[\timeidx]$ as an infinite superposition 
of signal components of the form $\mathbf{a}_{\theta} \exp(j \theta \timeidx)$ for $\theta \in [0,1)$ \cite{BachJordan04,Brockwell91}. The random coefficient vectors $\{\mathbf{a}_{\theta}\}_{\theta \in [0,1)}$, which are statistically independent over $\theta$, 
have zero mean and covariance matrix given by the SDM value $\SDM(\theta)$, i.e., $\expect \{ \mathbf{a}_{\theta} \mathbf{a}_{\theta}^{H} \} = \SDM(\theta)$.}

{
For our analysis, we require a mild technical condition for the eigenvalues $\lambda(\SDM(\theta))$ of the process SDM $\SDM(\theta)$. 
\begin{assumption}
\label{equ_asspt_bounded_eigvals}
For known positive constants $U \geq L>0$, we have 
\vspace*{-2mm}
\begin{equation}
\label{equ_uniform_bound_eigvals_specdensmatrix}
L \!\leq\! \lambda(\SDM(\theta)) \!\leq\! U  \mbox{ for every } \theta \!\in\! [0,1).
\vspace*{-1mm}
\end{equation} 
\end{assumption}
}
{We remark that the restriction induced by Assumption \ref{equ_asspt_bounded_eigvals} is rather weak. 
E.g., the upper bound in \eqref{equ_uniform_bound_eigvals_specdensmatrix} is already implied by the summability of the process ACF.} 
The lower bound in \eqref{equ_uniform_bound_eigvals_specdensmatrix} ensures that the CIG satisfies the global Markov property \cite{LauritzenGM,PHDEichler}. 
{An important and large class of processes satisfying \eqref{equ_uniform_bound_eigvals_specdensmatrix} is given by the set of stable VAR processes \cite{Luetkepol2005}.} 
In what follows, we will assume without loss of generality\footnote{For a stationary process $\mathbf{x}[\timeidx]$ whose SDM $\SDM(\theta)$ satisfies \eqref{equ_uniform_bound_eigvals_specdensmatrix}, with arbitrary {positive} constants $L$ and $U$, we can base our consideration equivalently on the scaled process $\mathbf{x}'[\timeidx] = \mathbf{x}[\timeidx]/\sqrt{L}$ whose SDM $\mathbf{S}_{x'}(\theta)$ satisfies \eqref{equ_uniform_bound_eigvals_specdensmatrix} with the 
constants $L' = 1$ and $U' = U/L$.} that $L=1$, implying that $U\geq1$. 

The CIG of the $\coefflen$-dimensional vector process $\vx[\timeidx]$ is the graph $\cig \defeq (\nodes,\edges)$ with node set $\nodes=[\coefflen]$, corresponding to the scalar process components $\{ x_{r}[\timeidx] \}_{r \in [\coefflen]}$, and edge set $\edges \subseteq \nodes \times \nodes$, defined by $(r,r') \notin \edges$ if the component processes $x_{r}[\timeidx]$ and $x_{r'}[\timeidx]$ are conditionally independent given 
the remaining components $\{ x_{t} [\timeidx] \}_{t \in [p] \setminus \{ r,r' \} }$ \cite{Dahlhaus2000}. 
{For a Gaussian stationary process $\mathbf{x}[\timeidx]$ whose SDM $\SDM(\theta)$ is invertible for every $\theta\! \in \! [0,1)$, which 
is implied by Assumption \ref{equ_asspt_bounded_eigvals}, the CIG of a process can be characterized conveniently via its SDM \cite{Dahlhaus2000,DahlhausEichler2003,Brillinger96remarksconcerning}:
\begin{lemma} 
Consider a Gaussian stationary vector process satisfying \eqref{equ_asspt_bounded_eigvals} and with associated CIG $\cig$ and SDM $\SDM(\theta)$. Then, 
two component processes $x_{r}[\timeidx]$ and $x_{r'}[\timeidx]$ are conditionally independent, given the remaining component processes $\{ x_{t}[\timeidx] \}_{t \in [p] \setminus \{r,r'\}}$, if and only if $\big( \SDM^{-1}(\theta) \big)_{r,r'}=0$ for all $\theta \in [0,1)$. Thus, the edge set $\edges$ of the CIG is characterized by 
\vspace*{-1mm} 
\begin{equation} 
\label{equ_charact_edge_set_covariance_matrix_k}
(r,r') \notin E \mbox{ if and only if } \big[\SDM^{-1}(\theta)\big]_{r,r'}\!=\!0 \quad \forall \theta \in [0,1).  
\vspace*{0mm}
\end{equation}
\end{lemma}
Thus, in the Gaussian case, the edge set $\edges$ corresponds to the zero entries of the inverse SDM $\SDM^{-1}(\theta)$, and the GMS problem is equivalent to detecting the zero entries of $\SDM^{-1}(\cdot)$. 
}

{We highlight that, by contrast to graphical models for random vectors, here we consider conditional independence relations between entire time series and not between scalar random variables. In particular, the CIG $\cig$ of a time series does not depend on time $\timeidx$ but applies to the entire time series. Let us illustrate this point by way of an example.}

{
Consider the vector autoregressive (VAR) process \cite{Luetkepol2005}  
\begin{equation} 
\label{equ_def_AR_process}
\mathbf{x}[\timeidx] = \mathbf{A} \mathbf{x}[\timeidx\!-\!1]\!+\!\mathbf{w}[\timeidx] \mbox{ with } \mathbf{A} = \begin{pmatrix} 0.5 & -0.5 \\ 0.5 & 0.5 \end{pmatrix}.
\end{equation} 
The noise process $\mathbf{w}[\timeidx]$ in \eqref{equ_def_AR_process} consists of i.i.d. Gaussian random vectors with zero mean and covariance matrix $\sigma^{2} \mathbf{I}$. 
Since the eigenvalues of the coefficient matrix $\mathbf{A}$ in \eqref{equ_def_AR_process}, given explicitly by $\exp(j \pi/4)/\sqrt{2}$ and $\exp(-j \pi/4)/\sqrt{2}$, have modulus strictly smaller than one, there exists 
a well defined stationary process $\vx[n]$ conforming to the recursion \eqref{equ_def_AR_process} (cf.\ \cite{Luetkepol2005}). A little calculation reveals that this stationary AR process has zero mean and its ACF is given by $\ACF[\lagvar] = \sigma^{2} \sum_{l=0}^{\infty} \mathbf{A}^{\lagvar+l} \big(\mathbf{A}^{T}\big)^{l}$ \cite{Luetkepol2005}. Since the 
VAR parameter matrix $\mathbf{A}$ in \eqref{equ_def_AR_process} satisfies $\mathbf{A}^{T} \mathbf{A} = (1/2) \mathbf{I}$, we have $\ACF[0] =  2 \sigma^{2} \mathbf{I}$. 
For an arbitrary but fixed time index $\timeidx\!=\!\timeidx_{0}$, the Gaussian random vector $\vx[\timeidx_{0}]$ is zero mean with covariance matrix $\mathbf{C} \!=\! \ACF[0]\!=\! 2 \sigma^{2} \mathbf{I}$. 
Thus, the scalar time samples $x_{1}[\timeidx]$ and $x_{2}[\timeidx]$ are marginally, i.e., for a fixed time index $\timeidx\!=\!\timeidx_{0}$, independent. 
However, since the inverse SDM of the process in \eqref{equ_def_AR_process} is given by \cite{Dahlhaus2000}
\begin{equation} 
\label{equ_expression_inv_SDM}
\SDM^{-1}(\theta) = \frac{1}{\sigma^{2}}\bigg[\begin{pmatrix} 1.5 & 0 \\ 0 & 1.5 \end{pmatrix} \!-\!  \begin{pmatrix} \cos \theta & j \sin \theta \\ -j \sin \theta & \cos \theta \end{pmatrix} \bigg].
\end{equation}
we have, upon comparing \eqref{equ_expression_inv_SDM} with the relation \eqref{equ_charact_edge_set_covariance_matrix_k}, that the entire scalar process components $\{ x_{1}[\timeidx]\}_{\timeidx \in \mathbb{Z}}$ and $\{x_{2}[\timeidx] \}_{\timeidx \in \mathbb{Z}}$ are dependent. 
In general, the marginal conditional independence structure at an arbitrary but fixed time $\timeidx\!=\!\timeidx_{0}$ is different from the conditional independence structure of the entire time series.
}

{The problem of GMS considered in this paper can be stated as that of inferring the CIG $\cig=(\nodes,\edges)$, or more precisely its edge set $\edges$, based on an observed finite length data block $\big( \mathbf{x}[1],\ldots,\mathbf{x}[\samplesize] \big)$. Similar to \cite{MeinBuhl2006}, our approach to GMS is to estimate the edge set $\edges$ by separately estimating the neighborhood $\mathcal{N}(r) \!\defeq\! \{  r' \in [\coefflen]  \,\, | (r,r') \in \edges \}$ of each node $r \in \nodes$. For the specific neighborhood $\mathcal{N}(1)$, the edge set characterization \eqref{equ_charact_edge_set_covariance_matrix_k} yields the following convenient characterization
\begin{equation}
\label{equ_charact_neigborhood_1_gsupp}
\mathcal{N}(1) = \gsupp \big( \big(\SDM(\cdot)\big)_{1,2:\coefflen} \big) - 1. 
\end{equation} 
The neighborhood characterization \eqref{equ_charact_neigborhood_1_gsupp} can be generalized straightforwardly to the neighborhood $\mathcal{N}(r)$ of an arbitrary node $r \!\in\! [p]$ (cf. Section \ref{sec_GMS_is_multitask_learning_problem}). 
For the derivation and analysis of the proposed GMS method, we will, besides Assumption \ref{equ_asspt_bounded_eigvals}, rely on three further assumptions on the CIG $\cig$, inverse SDM $\SDM^{-1}(\theta)$ and ACF $\ACF[\lagvar]$ of the underlying process $\vx[\timeidx]$.}

{
The first of these additional assumptions constrains the CIG of the observed process $\vx[\timeidx]$ to be \emph{sparse}, as made precise in 
\begin{assumption}
\label{asspt_max_node_degree}
The \emph{maximum node degree} $\max_{r \in [p]} | \mathcal{N}(r) |$ of the process CIG $\cig$ is upper bounded by a known small constant $s_{\text{max}}$, i.e., 
\vspace*{-2mm}
\begin{equation}
\label{equ_def_maximum_node_degree}
\max_{r \in [\coefflen]} | \mathcal{N}(r) | \leq s_{\emph{max}}   \ll \coefflen.
\vspace*{-1mm}
\end{equation} 
\end{assumption}}

{The next assumption is necessary} in order to allow for accurate selection schemes based on a finite length observation. {In particular, we} require that the non-zero entries of $\SDM^{-1}(\theta)$ are not too small. 
\begin{assumption}
For a known positive constant $\rho_{\emph{min}}$,  
\vspace*{-0mm}
\begin{equation} 
\label{equ_condition_beta_min_rho}
 \hspace*{-3mm}\min_{\substack{r \in [p] \\ r' \in \mathcal{N}(r)}}\hspace*{-2mm}  \left(  \int_{\theta=0}^{1} \big| \big[  \SDM^{-1}(\theta) \big]_{r,r'} / \big[ \SDM^{-1}(\theta)  \big]_{r,r} \big|^{2} d\theta \right)^{1/2} 
\!\geq\!  \rho_{\emph{min}}.
\vspace*{-1mm} 
\end{equation}
\end{assumption}
Note that the integrand in \eqref{equ_condition_beta_min_rho} is well defined, since by \eqref{equ_uniform_bound_eigvals_specdensmatrix} we have $\big[ \SDM^{-1}(\theta)  \big]_{r,r} \geq (1/U) > 0$ for all $\theta \in [0,1)$ and any $r \in [\coefflen]$. 
If, for some positive $\rho_{\text{min}}>0$, \eqref{equ_condition_beta_min_rho} is in force, \eqref{equ_charact_edge_set_covariance_matrix_k} becomes: 
$(r,r')\! \notin\! \edges$ if and only if $\big\| \big[\SDM^{-1}(\cdot)\big]_{r,r'} \big\|_{2} \!=\!0$.
 
{By contrast to existing approaches to GMS for time series, we do not assume a finite parametric model for the observed process.}
However, for the proposed selection method to be accurate, we require the process $\mathbf{x}[\timeidx]$ to be sufficiently smooth in the spectral domain. 
By \emph{a smooth process} $\mathbf{x}[\timeidx]$, we mean a process $\vx[\timeidx]$ such that the entries of its SDM $\SDM(\theta)$ are smooth functions of $\theta$. These smoothness 
constraints will be expressed in terms of moments of the process ACF:
\begin{assumption} 
\label{asspt_ACF_moment}
For a small positive constant $\mu_{0}$ and a given non-negative weight function $h[\lagvar]$, that typically increases with $|\lagvar|$, we have the 
bound 
\vspace*{-1mm}
\begin{equation}
\label{equ_def_generic_moments_ACF}
\mu^{(h)}_{x} \defeq \sum_{\lagvar=-\infty}^{\infty} h[\lagvar]  \| \autocovfunc[\lagvar] \|_{\infty} \leq \mu_{0}. 
\vspace*{-1mm}
\end{equation}
\end{assumption}
For the particular weighting function 
$h[\lagvar] \defeq |\lagvar|$, we will use the shorthand  
\begin{equation}
\label{equ_conditoin_ACF_moment_times_mag_m}
 \mu_{x} \defeq \sum_{\lagvar=-\infty}^{\infty} |\lagvar| \| \autocovfunc[\lagvar] \|_{\infty}.
\end{equation}
{We may interpret the moment $\mu_{x}$ as a measure for the effective ACF width of the process.}

Another particular choice for the weighting function will be given in Section \ref{sec_var_sel_consist}. This choice is related to the window function of the BT estimator which is 
part of our GMS method (cf. Section \ref{sec_novel_sel_scheme}). 

{We note that Assumption \ref{asspt_ACF_moment} is similar in spirit to the underspread assumption for linear time varying systems and nonstationary processes \cite{MatzHla06} in that it allows to construct efficient decorrelation transformations. In particular, for a smooth process conforming to Assumption \ref{asspt_ACF_moment}, one can verify that the discrete Fourier transform (DFT) of the observed block yields random vectors which are approximately uncorrelated for different frequencies. This decorrelation in the frequency domain is the key idea behind our Fourier based approach.}

In what follows, we will formulate and analyze a GMS scheme for the class of $\coefflen$-dimensional {Gaussian stationary processes $\vx[\timeidx]$ conforming to Assumptions \ref{equ_asspt_bounded_eigvals}-\ref{asspt_ACF_moment}}. This process class will be denoted as {$\mathcal{M}$ for brevity.}

\section{The Selection Scheme} 
\label{sec_novel_sel_scheme}

The GMS scheme developed in this section is inspired by the neighborhood regression approach in \cite{MeinBuhl2006}. A main conceptual difference of our approach to \cite{MeinBuhl2006} is that we perform neighborhood  regression in the frequency domain. Moreover, while the approach in \cite{MeinBuhl2006} is based on a standard sparse linear regression model, we formulate the neighborhood regression for time series as a multitask learning problem. {This multitask learning problem is based on 
an estimator for the SDM, which will be discussed next.}

\subsection{SDM Estimation} 

{Due to the direct relation \eqref{equ_charact_edge_set_covariance_matrix_k} between the zero pattern of the inverse SDM and the edge set of the CIG,} a naive approach to GMS would be to first estimate the SDM, then invert this estimate and determine the location of the 
non-zero entries. 
With regards to the first step, it is natural to estimate $\SDM(\theta)$ by replacing the ACF in \eqref{equ_def_spectral_density_matrix} with  an empirical version $\EACF[\lagvar]$ which is based on sample averages. This yields the estimate 
\begin{align} 
\label{equ_est_BT_sdm_theta_t}
\ESDM(\theta) & \defeq  \sum_{\lagvar=-\samplesize+1}^{\samplesize-1} w[\lagvar]\EACF[\lagvar]  e^{-j 2 \pi \theta \lagvar }
\end{align} 
where, $\EACF^{T}[-\lagvar] = \EACF[\lagvar]$ and
\begin{equation}
\label{equ_def_ACF_est_BT}
\EACF[\lagvar] \!\defeq\!  \frac{1}{\samplesize} \! \sum_{\timeidx=1}^{\samplesize-\lagvar} \! \vx[\timeidx+\lagvar] \transp{\vx}[\timeidx]  \mbox{, } \lagvar\!\in\!\{0,\ldots,\samplesize\!-\!1\}.
\end{equation}
{Note that the SDM estimator \eqref{equ_est_BT_sdm_theta_t} can be regarded as the natural adaptation, to the case of 
SDM estimation for vector process, of the BT estimator \cite{stoi97} for the 
power spectral density of a scalar process.}

The real-valued window function $w[\lagvar]$ in \eqref{equ_est_BT_sdm_theta_t}, from now on assumed to satisfy 
\begin{equation}
\label{equ_finite_support_window_w}
w[\lagvar] = 0 \mbox{ for } \lagvar \geq \samplesize \mbox{ and } w[0]=1, 
\end{equation} 
is chosen such that the estimate $\ESDM(\theta)$ is guaranteed to be a psd matrix. A sufficient condition for this to be the case is non-negativity of  
the discrete-time Fourier transform (DTFT) $W(\theta)$ of the window function, i.e., 
$W(\theta) \defeq \sum_{\lagvar = - \infty}^{\infty} w[\lagvar] \exp(- j 2 \pi \theta \lagvar) \geq 0$ \cite[p. 40]{stoi97}. 

In what follows, we need a specific representation of the estimate $\ESDM(\theta)$ in \eqref{equ_est_BT_sdm_theta_t}, which is stated in 
\begin{lemma}
\label{lem_repr_ESDM_factor}
Consider the estimate $\ESDM(\theta)$ given by \eqref{equ_est_BT_sdm_theta_t}, for $\theta \in [0,1)$. 
{Let us define the matrix 
\begin{equation}
\label{equ_est_BT_factorization_def_A}
\mathbf{A}(\theta)\!\defeq\!  \sqrt{\mathbf{W}(\theta)} \mathbf{F}^{T} \mathbf{D}^{T} , 
\end{equation}
where $\mathbf{D} \defeq \begin{pmatrix} \mathbf{x}[1],\ldots,\mathbf{x}[\samplesize] \end{pmatrix} \in \mathbb{R}^{\coefflen \times  \samplesize}$
is the data matrix, $\mathbf{F} \in \mathbb{C}^{\samplesize \times (2\samplesize\!-\!1)}$ denotes the first $\samplesize$ rows of the size-$(2\samplesize\!-\!1)$ DFT matrix, i.e., 
$\left( \mathbf{F} \right)_{k,l} = \exp(- 2 \pi (k\!-\!1) (l\!-\!1) /(2\samplesize\!-\!1))$ and 
\begin{equation} 
\label{equ_def_diagonal_matrix_weight_function} 
\mathbf{W}(\theta) \!\defeq\!  \frac{1}{2\samplesize\!-\!1} \diag\{ W(\theta_{\freqbin}\!+\!\theta) \}_{\freqbin \in [2\samplesize\!-\!1]},
\end{equation} 
with $\theta_{\freqbin} \defeq 2\pi(\freqbin\!-\!1)/(2\samplesize\!-\!1)$.}

{
We then have the identity  
\begin{equation}
\label{equ_est_BT_factorization}
\ESDM(\theta)\! = \! (1/\samplesize)  \herm{\mathbf{A}}(\theta) \mathbf{A}(\theta). 
\end{equation}}
\end{lemma}
\begin{proof}
Appendix \ref{lem_repr_ESDM_factor_proof}.
\end{proof}

As evident from the factorization \eqref{equ_est_BT_factorization}, the rank of $\ESDM(\theta)$ satisfies 
$\rank \{ \ESDM(\theta) \} \leq \samplesize$.
Therefore, in the high-dimensional regime, where the number $\samplesize$ of observations may be much smaller than the number $\coefflen$ of process components, the estimates $\ESDM(\theta) \in \mathbb{C}^{\coefflen \times \coefflen}$ 
will {typically} be rank-deficient and thus cannot be inverted to obtain estimates  
of the edge set $\edges$ via the relation \eqref{equ_charact_edge_set_covariance_matrix_k}. 

In order to cope with the rank deficiency of the SDM estimate $\ESDM(\theta)$, we next show that finding the support of the inverse SDM $\SDM^{-1}(\theta)$ based on the observation $\vx[1],\ldots,\vx[\samplesize]$ can be formulated as a \emph{multitask learning problem}. {For clarity, we detail this approach only for the problem of estimating the neighborhood $\mathcal{N}(1)$. 
The generalization to the estimation of the neighborhood $\mathcal{N}(r)$ of an arbitrary node $r \in [p]$ is straightforward.} 

{
Indeed, consider the permuted process $\tilde{\mathbf{x}}[\timeidx] \defeq  \mathbf{P}_{r} \mathbf{x}[\timeidx]$, with 
the permutation matrix $\mathbf{P}_{r} \!\defeq\! \big( \mathbf{e}_{\Pi_{r}(1)},\ldots,\mathbf{e}_{\Pi_{r}(\coefflen)} \big)$ where $\Pi_{r}(\cdot):[\coefflen] \rightarrow [\coefflen]$ denotes the permutation exchanging entry $1$ with entry $r$. As can be verified easily, the SDM $\mathbf{S}_{\tilde{\mathbf{x}}}(\theta)$ of the process $\tilde{\mathbf{x}}[\timeidx]$ is then given by $\mathbf{P}_{r} \SDM(\theta) \mathbf{P}_{r}$. Moreover, the CIG $\mathcal{G}_{\tilde{x}}$ of $\tilde{\mathbf{x}}[\timeidx]$ contains the edge $(v,w)$ if and only if the CIG $\mathcal{G}_{x}$ of $\mathbf{x}[\timeidx]$ 
contains the edge $(\Pi_{r}(v),\Pi_{r}(w))$, i.e., 
\begin{equation}
\label{equ_relation_CIG_permuted}
(v,w) \in \mathcal{G}_{\tilde{x}} \mbox{ if and only if } (\Pi_{r}(v),\Pi_{r}(w)) \in \mathcal{G}_{x}.
\end{equation}  
Thus, the problem of determining the neighborhood $\mathcal{N}(r)$ in the CIG of the process $\vx[\timeidx]$ is equivalent to the problem of determining the 
neighborhood $\mathcal{N}(1)$ in the CIG of the permuted process $\tilde{\mathbf{x}}[\timeidx] =  \mathbf{P}_{r}\mathbf{x}[\timeidx]$.}

\subsection{Multitask Learning Formulation}
\label{sec_GMS_is_multitask_learning_problem}

{The basic intuition behind our approach is to perform a decorrelation of the time samples $\vx[1],\ldots,\vx[\samplesize]$ by applying a DFT.  
In particular, given the observation $\mathbf{D}\!=\!(\mathbf{x}[1],\ldots,\mathbf{x}[\samplesize]) \in \mathbb{R}^{\coefflen \!\times\!\samplesize}$, we compute the length-$(2\samplesize\!-\!1)$ DFT as 
\begin{equation}
\label{equ_def_DFT}
\hat{\mathbf{x}}[\freqbin] \!\defeq\! \frac{1}{\sqrt{\samplesize}} \sum_{n\in [\samplesize]} \mathbf{x}[n] \exp( - j 2 \pi (n\!-\!1)(\freqbin\!-\!1)/(2\samplesize\!-\!1)), 
\end{equation}
for $\freqbin \in [2\samplesize\!-\!1]$. It can be shown that for a vector process $\vx[\timeidx]$ conforming to Assumption \ref{asspt_ACF_moment} and a sufficiently large sample size $\samplesize$, the DFT vectors $\hat{\mathbf{x}}[1],\ldots,\hat{\vx}[2\samplesize\!-\!1]$, which may be interpreted as random samples 
indexed by frequency $\freqbin$, are approximately independent. However, what hinders the straight application of the neighborhood regression method in \cite{MeinBuhl2006}, developed for the case of i.i.d. samples, is the fact 
that the samples $\hat{\mathbf{x}}[\freqbin]$ are not identically distributed. Indeed, the covariance matrix of the Gaussian random vector $\hat{\mathbf{x}}[\freqbin]$ is roughly equal to the SDM value $\SDM(\theta_{\freqbin}\!=\!2\pi(\freqbin\!-\!1)/(2\samplesize\!-\!1))$, which in general varies with $\freqbin$. However, for processes with a smooth SDM, i.e., conforming to Assumption \ref{asspt_ACF_moment} with small $\mu_{0}$, the SDM is approximately constant over small frequency intervals and therefore, in turn, the distribution of consecutive samples $\hat{\mathbf{x}}[\freqbin]$ is nearly identical. We exploit this by masking the DFT samples $\hat{\mathbf{x}}[\freqbin]$ such that, for a given center frequency $\theta \in [0,1)$, we only retain those samples $\hat{\mathbf{x}}[\freqbin]$ which fall into the pass band of the spectral window $W(\theta_{\freqbin}\!+\!\theta)$ in \eqref{equ_def_diagonal_matrix_weight_function}, which is the shifted (by the center frequency $\theta$) DTFT of the window function $w[\lagvar]$ employed in the BT estimator \eqref{equ_est_BT_sdm_theta_t}. This spectral masking then yields the modified DFT samples
\begin{equation}
\label{equ_def_masked_DFT}
\tilde{\mathbf{x}}^{(\theta)}[\freqbin] \defeq \sqrt{W(\theta_{\freqbin}\!+\!\theta)} \hat{\mathbf{x}}[\freqbin] \mbox{, for } \freqbin \in [2\samplesize\!-\!1]. 
\end{equation}
By considering the significant DFT vectors $\tilde{\mathbf{x}}[\freqbin]$ approximately as i.i.d. samples of a Gaussian random vector with zero mean and covariance matrix $\SDM(\theta)$, we can 
immediately apply the neighborhood regression approach in \cite{MeinBuhl2006}. In particular, we formulate, for a specific center frequency $\theta$, a sparse linear regression problem by regressing 
the first entry of the vector $\tilde{\mathbf{x}}^{(\theta)}[\freqbin]$ against its remaining entries. More precisely, based on the vector $\vy(\theta) \in \mathbb{C}^{2\samplesize\!-\!1}$ and matrix $\mathbf{X}(\theta) \in \mathbb{C}^{(2\samplesize\!-\!1) \times (\coefflen-1)}$, 
\begin{equation}
\label{equ_def_vec_y_mtx_X}
\vy(\theta) \!\defeq\! \begin{pmatrix}  \tilde{x}_{1}^{(\theta)}[1] \\ \vdots \\Ê\tilde{x}_{1}^{(\theta)}[2\samplesize\!-\!1] \end{pmatrix} \mbox{, } \mathbf{X}(\theta) \!\defeq\! \begin{pmatrix}  \big( \tilde{\mathbf{x}}_{2:\coefflen}^{(\theta)}[1] \big)^{T}  \\ 
\vdots \\ \big( \tilde{\mathbf{x}}_{2:\coefflen}^{(\theta)}[2\samplesize\!-\!1] \big)^{T} \end{pmatrix},  
\end{equation} 
we define, for each $\theta\in [0,1)$, the linear regression model 
\vspace*{-2mm}
\begin{equation}
\label{equ_def_signal_model_given_cov_matrix}
\vy(\theta) \defeq \mX(\theta) \vbe(\theta) + \vep(\theta) 
\vspace*{-1mm}
\end{equation} 
with the parameter vector 
\vspace*{-2mm}
\begin{equation}
\label{equ_def_parameter_vector_given_cov_matrix}
\vbe(\theta) \defeq  \big[ \big(\SDM(\theta)\big)_{2:\coefflen,2:\coefflen} \big]^{-1}  (\SDM(\theta))_{2:\coefflen,1}.
\vspace*{-1mm}
\end{equation}
Let us make the relation between the quantities $\mathbf{y}(\theta)$, $\mathbf{X}(\theta)$ and the observed data $\mathbf{D}=(\vx[1],\ldots,\vx[\samplesize])$ more explicit by noting that, upon combining \eqref{equ_def_DFT} with \eqref{equ_def_masked_DFT} and inserting into \eqref{equ_def_vec_y_mtx_X}, we have 
\begin{equation}
\label{equ_explicit_extression_y_DFT}
\vy(\theta) \!=\!  \sqrt{\mathbf{W}(\theta)} \mathbf{F}^{T} \big( \mathbf{D}_{1,1:\samplesize}\big)^{T} 
\end{equation} 
and 
\begin{equation}
\label{equ_explicit_extression_X_DFT}
\mX (\theta) \!=\! \sqrt{\mathbf{W}(\theta)} \mathbf{F}^{T} \big(\mathbf{D}_{2:\coefflen,1:\samplesize}\big)^{T}.
\end{equation} 
Note that the product $ \mathbf{F}^{T} \big( \mathbf{D}_{1,1:\samplesize}\big)^{T}$ in \eqref{equ_explicit_extression_y_DFT} just amounts to computing the DFT (of length $2\samplesize\!-\!1$) of the process component $x_{1}[\timeidx]$. 
Similarly, the rows of $\mathbf{F}^{T} \big(\mathbf{D}_{2:\coefflen,1:\samplesize}\big)^{T}$ in \eqref{equ_explicit_extression_X_DFT} are given by the DFTs (of length $2\samplesize\!-\!1$) of the process components $x_{2}[\timeidx],\ldots,x_{\coefflen}[\timeidx]$. 
}

{The error term $\vep(\theta)$ in \eqref{equ_def_signal_model_given_cov_matrix} is defined implicitly via the definitions \eqref{equ_def_parameter_vector_given_cov_matrix}, \eqref{equ_explicit_extression_y_DFT}, and \eqref{equ_explicit_extression_X_DFT}.
It will be shown in Section \ref{sec_var_sel_consist} that, if the SDM estimator \eqref{equ_est_BT_sdm_theta_t} is accurate, i.e., $\ESDM(\theta)$ is close to $\SDM(\theta)$ uniformly for all $\theta \in [0,1)$, the error term $\vep(\theta)$ will be small.}

{As can be verified easily, by comparing expressions \eqref{equ_explicit_extression_y_DFT} and \eqref{equ_explicit_extression_X_DFT} with \eqref{equ_est_BT_factorization}, the vector $\vy(\theta)$ and the matrix $\mX (\theta)$ are given by the columns of the matrix $\mathbf{A}(\theta)$ in \eqref{equ_est_BT_factorization_def_A} of Lemma \ref{lem_repr_ESDM_factor}.
Therefore,} according to \eqref{equ_est_BT_factorization}, we have the identity 
\vspace*{0mm}
\begin{equation}
\label{equ_grammian_of_mult_learning_problem_ESDM}
\begin{pmatrix}  \mathbf{y}(\theta) \,\,\, \mathbf{X}(\theta) \end{pmatrix}^{H} \begin{pmatrix}  \mathbf{y}(\theta) \,\,\, \mathbf{X}(\theta) \end{pmatrix} = \ESDM(\theta) \mbox{, for } \theta\!\in\![0,1), 
\vspace*{0mm}
\end{equation} 
{where $ \ESDM(\theta)$ denotes the BT estimator in \eqref{equ_est_BT_sdm_theta_t}.}

{The link between the multitask learning problem \eqref{equ_def_signal_model_given_cov_matrix} and the problem of determining the neighborhood 
$\mathcal{N}(1)$ is stated in 
\begin{lemma} 
Consider the parameter vector $\vbe(\theta)$ defined for each $\theta \in [0,1)$ via \eqref{equ_def_parameter_vector_given_cov_matrix}. The generalized support of $\vbe(\cdot)$ is related to $\mathcal{N}(1)$ via
\begin{equation}
\gsupp(\vbe(\cdot)) = \mathcal{N}(1)\!-\!1.
\label{eq:suppeqbeta}
\end{equation}
\end{lemma} 
\begin{proof}
Let us partition the SDM $\SDM(\theta)$ and its inverse $\SDM^{-1}(\theta)$ as 
\vspace*{-1mm}
\begin{equation}
\label{equ_convenient_partitioning_SDM}
\hspace*{-4mm}\begin{pmatrix} \gamma(\theta) & \herm{\mathbf{c}}(\theta) \\[2mm]  \mathbf{c}(\theta) & \mathbf{G}(\theta) \end{pmatrix} \! \defeq \! \SDM(\theta)  \mbox{, }  \begin{pmatrix} \tilde{\gamma}(\theta) & \herm{\tilde{\mathbf{c}}}(\theta) \\[2mm]  \tilde{\mathbf{c}}(\theta) & \widetilde{\mathbf{G}}(\theta) \end{pmatrix} \! \defeq \! \SDM^{-1}(\theta).
\vspace*{-1mm}
\end{equation} 
According to \eqref{equ_charact_edge_set_covariance_matrix_k}, 
we have 
\vspace*{-3mm}
\begin{equation}
\label{equ_identity_gsupp_c_neighborhood}
\gsupp(\tilde{\mathbf{c}}(\cdot)) \stackrel{\eqref{equ_charact_edge_set_covariance_matrix_k}}{=} \mathcal{N}(1)\!-\!1,
\vspace*{0mm}
\end{equation} 
where $\tilde{\mathbf{c}}(\theta)$ is the lower left block of $\SDM^{-1}(\theta)$ (cf.\ \eqref{equ_convenient_partitioning_SDM}), which is seen as follows. 
Applying a well known formula for the inverse of a block matrix (cf.\ \cite[Fact 2.17.3 on p.\ 159]{bernstein09}) to the partitioning \eqref{equ_convenient_partitioning_SDM}, 
\begin{equation}
\label{equ_matrix_inversion_block_mtx_parameter_vector}
\tilde{\mathbf{c}}(\theta) = - \tilde{\gamma}(\theta) \mathbf{G}^{-1}(\theta) \mathbf{c}(\theta) 
\stackrel{\eqref{equ_def_parameter_vector_given_cov_matrix}}{=} -\vbe(\theta) \tilde{\gamma}(\theta).
\end{equation}
Note that $\tilde{\gamma}(\theta) \stackrel{\eqref{equ_convenient_partitioning_SDM}}{=} \big[  \SDM^{-1}(\theta)\big]_{1,1}>0$, since we assume $\SDM(\theta)$ to be strictly positive definite (cf.\ \eqref{equ_uniform_bound_eigvals_specdensmatrix}), implying in turn that $\SDM^{-1}(\theta)$ is also strictly positive definite. 
Therefore,
\begin{equation} 
\gsupp({\bm \beta}(\cdot)) \stackrel{\eqref{equ_matrix_inversion_block_mtx_parameter_vector}}{=} \gsupp(\tilde{\mathbf{c}}(\cdot)) \stackrel{\eqref{equ_identity_gsupp_c_neighborhood}}{=} 
 \mathcal{N}(1)\!-\!1. \nonumber
\end{equation} 
\end{proof}}
Thus, the problem of determining the neighborhood $\mc N(1)$ of node $r\!=\!1$ has been reduced to that of finding the joint support of the parameter vectors $\{\vbe(\theta)\}_{\beta \in [0,1)}$ from the observation of the vectors $\{\vy(\theta)\}_{\theta \in [0,1)}$ given by \eqref{equ_def_signal_model_given_cov_matrix}. 

Recovering the vector ensemble $\{ \vbe(\theta) \}_{\theta \in [0,1)}$ with a small generalized support from the linear measurements $\{\vy(\theta)\}_{\theta \in [0,1)}$, given by \eqref{equ_def_signal_model_given_cov_matrix}, is an instance of a \emph{multitask learning problem} \cite{Lounici09,Argyriou07convexmulti-task,BuhlGeerBook,ObozWainJor11}, being, in turn, a special case of a block-sparse recovery problem \cite{BlockSparsityEldarTSP}. 
Compared to {existing work on multitask learning} \cite{Lounici09,Argyriou07convexmulti-task,BuhlGeerBook,ObozWainJor11}, the distinctive feature of the multitask learning problem given by \eqref{equ_def_signal_model_given_cov_matrix} is that we have a continuum of individual tasks indexed by $\theta \in [0,1)$. 
The closest to our setting is \cite{MishaliEldar2009,MishaliEldar2008}, where also multitask learning problems with a continuum of tasks have been considered. However, the authors of \cite{MishaliEldar2009,MishaliEldar2008} 
require the system matrix $\mX(\theta)$ to be the same for all tasks. 
To the best of our knowledge, general multitask learning problems with a continuum of tasks of the form \eqref{equ_def_signal_model_given_cov_matrix} have not been considered so far. 

\subsection{Multitask LASSO based GMS}
A popular approach for estimating a set of vectors with a {small joint support, based on linear measurements,} is the \emph{group LASSO} \cite{YuanLin2006}. 
Specializing the group LASSO to the multitask model \eqref{equ_def_signal_model_given_cov_matrix} yields the \emph{multitask LASSO} (mLASSO) \cite{BuhlGeerBook,Lee_adaptivemulti-task}. 
However, while \cite{BuhlGeerBook,Lee_adaptivemulti-task} {consider a finite number of tasks}, we consider a continuum of tasks 
indexed by the frequency $\theta \in [0,1)$. An obvious generalization of the mLASSO to our setting is 
\vspace*{0mm}
\begin{equation}
\label{equ_def_multitask_LASSO}
\hspace*{-2mm} \hat{\vbe}[\vy(\cdot),\mX(\cdot)]  \!\defeq\!  \hspace*{-2mm} \argmin_{ \vbed  \in \ell_{q}([0,1))} \big\| \vy(\cdot)\!-\!\mX(\cdot) \vbed(\cdot) \big\|_{2}^{2}  \!+\! \lambda \| \vbed(\cdot) \|_{1}. 
\end{equation} 
If the design parameter $\lambda\!>\!0$ in \eqref{equ_def_multitask_LASSO} is chosen suitably (cf.\ Section \ref{sec_var_sel_consist}), the generalized support of $\hat{\vbe}(\cdot)$ 
coincides with that of  the true parameter vector $\vbe(\cdot)$ in \eqref{equ_def_signal_model_given_cov_matrix}, i.e., 
\begin{equation}
\label{equ_relation_gsupp_hat_vbe_neighbor}
\gsupp( \hat{\vbe}(\cdot)) = \gsupp( \vbe (\cdot) ) \stackrel{\eqref{eq:suppeqbeta}}{=} \mathcal{N}(1)- 1. 
\end{equation} 

{Thus, we can determine the neighborhood $\mathcal{N}(1)$ via computing the mLASSO based on the observation vector and system matrix constructed 
via \eqref{equ_explicit_extression_y_DFT} and \eqref{equ_explicit_extression_X_DFT} from the observed data $\mathbf{D}=\big(\vx[1],\ldots,\vx[\samplesize]\big)$. The generalization to the determination of the neighborhood $\mathcal{N}(r)$ for an arbitrary node $r\in[\coefflen]$ is acomplished via \eqref{equ_relation_CIG_permuted} by using the permuted observation $\widetilde{\mathbf{D}} \!\defeq\! \mathbf{P}_{r} \big(\vx[1],\ldots,\vx[\samplesize]\big)$ in \eqref{equ_explicit_extression_y_DFT} and \eqref{equ_explicit_extression_X_DFT} instead of $\mathbf{D}$. We arrive at the following algorithm for estimating the CIG of the observed process.}
\begin{algorithm}
\label{alg:1}

\begin{enumerate}
\vspace*{3mm}
\item {Given a specific node $r \in [\coefflen]$, form the permuted data matrix $\widetilde{\mathbf{D}} \!=\! \mathbf{P}_{r} \big(\vx[1],\ldots,\vx[\samplesize]\big) $, 
and compute the observation vector $\vy(\theta)$ and system matrix 
$\mX(\theta)$ according to 
\begin{equation} 
\label{equ_explicit_extression_y_DFT_tilde}
\vy(\theta) \!=\!  \sqrt{\mathbf{W}(\theta)} \mathbf{F}^{T} \big(\widetilde{\mathbf{D}}_{1,1:\samplesize}\big)^{T} 
\end{equation} 
and 
\begin{equation}
\label{equ_explicit_extression_X_DFT_tilde}
\mX (\theta) \!=\! \sqrt{\mathbf{W}(\theta)} \mathbf{F}^{T} \big( \widetilde{\mathbf{D}}_{2:\coefflen,1:\samplesize}\big)^{T}.
\end{equation} 
}

\vspace{3mm}

\item {Based on the observation vector $\vy(\theta)$ and system matrix $\mX(\theta)$ given by \eqref{equ_explicit_extression_y_DFT_tilde} and \eqref{equ_explicit_extression_X_DFT_tilde},} compute the mLASSO estimate $\hat{{\bm \beta}}(\theta)$ according to \eqref{equ_def_multitask_LASSO} and  
estimate the neighborhood $\mathcal{N}(r)$ by the index set 
\vspace*{-1mm}
\begin{equation}
\label{equ_def_est_neighborhood_LASSO_CS_sel}
\widehat{\mathcal{N}}(r) = \{\Pi_{r}(r'+1)  \, | \, r' \in [p] \mbox{, }   \norm[L^2]{\hat{\beta}_{r'}(\cdot)}  > \eta\}, 
\vspace*{-1mm}
\end{equation}
for some suitably chosen threshold $\eta$. 

\vspace*{3mm} 

\item {Repeat step $1)$ and step $2)$ for all nodes $r \!\in\! [\coefflen]$ and combine the individual neighborhood estimates $\widehat{\mathcal{N}}(r)$ 
to obtain the final CIG estimate $\widehat{\mathcal{G}}=([\coefflen],\widehat{\edges})$.}
\end{enumerate}
\end{algorithm}
The proper choice for the mLASSO parameter $\lambda$ in \eqref{equ_def_multitask_LASSO} and the threshold $\eta$ in \eqref{equ_def_est_neighborhood_LASSO_CS_sel} will be discussed in Section \ref{sec_var_sel_consist}.

{
For the last step of Algorithm \ref{alg:1}, different ways of combining the individual neighborhood estimates $\widehat{\mathcal{N}}(r)$ to obtain the 
edge set of the CIG estimate $\widehat{\mathcal{G}}$ are possible. Two intuitive choices are the ``AND'' rule and the ``OR'' rule. 
For the AND (OR) rule, an edge $(r,r')$ is present in $\widehat{\mathcal{G}}$, i.e. $(r,r') \in \widehat{\edges}$, if and only if $r \in \widehat{\mathcal{N}}(r')$ and (or) $r' \in \widehat{\mathcal{N}}(r)$.} 

Note that the optimization in \eqref{equ_def_multitask_LASSO} has to be carried out over the Hilbert space $\ell_{q}([0,1))$ with inner product 
$\langle \mathbf{f}(\cdot), \mathbf{g}(\cdot) \rangle_{\ell_{q}} \defeq  \int_{\theta=0}^{1} \mathbf{g}^{H}(\theta) \mathbf{f}(\theta) d \theta$, and induced norm $\|Ê\mathbf{g}(\cdot) \|_{2} = \sqrt{\sum_{r} \|g_{r}(\cdot)\|_{L^{2}}}$.
Since the cost function in \eqref{equ_def_multitask_LASSO} is convex, continuous and coercive, i.e., $\lim\limits_{\|\vbe(\cdot)\| \rightarrow \infty} f[\vbe(\cdot)]Ê\!\rightarrow \! \infty$, 
it follows by convex analysis that a minimizer for \eqref{equ_def_multitask_LASSO} exists \cite{BossavitConvexAnal}. In the case of multiple solutions, we mean by $\hat{\vbe}(\cdot) =  \argmin_{ \vbed(\cdot) \in \ell_{q}([0,1))} f[\vbed(\cdot)]$ any of these solutions.\footnote{{Note that a sufficient condition for uniqueness of the solution to \eqref{equ_def_multitask_LASSO} would be strict convexity of the objective function. However, in the high-dimensional regime where $\samplesize \ll \coefflen$ the 
system matrix $\mathbf{X}(\theta) \in \mathbb{C}^{(2\samplesize\!-\!1)\times (\coefflen\!-\!1)}$ defined by \eqref{equ_explicit_extression_X_DFT_tilde} is singular and therefore the 
objective function in \eqref{equ_def_multitask_LASSO} is not strictly convex. Thus, in this regime, uniqueness of the solution to \eqref{equ_def_multitask_LASSO} requires additional assumptions such as, e.g., incoherence conditions \cite{BachConsistency2008,NegWain2011}. We emphasize, however, that our analysis does not require uniqueness of the solution to \eqref{equ_def_multitask_LASSO}.}}  

{Let us finally mention that, in principle, Algorithm \ref{alg:1} can also be applied to non-Gaussian processes. However, 
the resulting graph estimate $\widehat{\mathcal{G}}$ is then not related to a CIG anymore but to a \emph{partial correlation graph} of the process \cite{Dahlhaus2000}. 
By contrast to a CIG, which is based on the exact probability distribution of the process, a partial correlation graph encodes only the second order statistic, i.e., the partial correlation structure of a vector process. 
In the Gaussian case, however, these two concepts coincide. 
}

\subsection{Numerical Implementation}

In order to numerically solve the optimization problem \eqref{equ_def_multitask_LASSO} we will use a simple discretization approach. More precisely, we require the optimization variable 
$\vbe(\cdot) \in \ell_{q}([0,1))$ to be piecewise constant over the frequency intervals $[(\task\!-\!1)/\nrtasks,\task/\nrtasks)$, for $\task \in [\nrtasks]$, where the number 
$\nrtasks$ of intervals is chosen sufficiently large. As a rule of thumb, we will use $\nrtasks \approx 2 \mu_{x}$, since the SDM $\SDM(\theta)$ is approximately constant over frequency intervals smaller than $1/\mu_{x}$. This may be 
verified by the Fourier relationship \eqref{equ_def_spectral_density_matrix} {between the process SDM and ACF}. Thus, if we denote by $I_{\task}(\theta)$ the 
indicator function of the frequency interval $[(\task\!-\!1)/\nrtasks,\task/\nrtasks)$, we represent the optimization variable $\vbe(\cdot) \in \ell_{q}([0,1))$ as 
\begin{equation}
\label{equ_repr_opt_var_piecewise_constant}
\vbe(\theta) = \sum_{\task \in [\nrtasks]} Ê\vbe_{\task} I_{\task}(\theta),  
\end{equation}
{with the vector-valued expansion coefficients $\vbe_{\task}  \in \mathbb{C}^{q}$.}
Inserting \eqref{equ_repr_opt_var_piecewise_constant} into \eqref{equ_def_multitask_LASSO} yields the finite-dimensional mLASSO 
\begin{equation} 
\label{equ_discretized_mLASSO}
\hat{\vbe} \!=\! \argmin_{\vbe  =(\vbe_{1},\ldots,\vbe_{\nrtasks})^{T}} \hspace*{-2mm} \sum_{\task \in [\nrtasks]}\vbe^{H}_{\task} \mathbf{G}_{\task} \vbe_{\task}\! -\! 2 \Re \{ \mathbf{c}^{H}_{\task} \vbe_{\task} \} +  \lambda \| \vbe \|_{1} 
\end{equation} 
with $\mathbf{G}_{\task} \defeq \int_{\theta=(\task-1)/\nrtasks}^{\task/\nrtasks}  \mX^{H}(\theta)\mX(\theta) d \theta$ and $\mathbf{c}_{\task} \defeq \int_{\theta=(\task-1)/\nrtasks}^{\task/\nrtasks}   \mX^{H}(\theta)\vy(\theta) d \theta$. Here, we used $\| \vbe \|_{1} \defeq \sum_{r \in [q]} \| \vbe^{(r)} \|_{2}$ with {the vectors $\vbe^{(r)} \in \mathbb{C}^{\nrtasks}$ given elementwise as} $\big(  \vbe^{(r)} \big)_{\task} \defeq \big( \vbe_{\task} \big)_{r}$. Based on the solution $\hat{\vbe} = \big(\hat{\vbe}_{1},\ldots,\hat{\vbe}_{\nrtasks} \big)$ of \eqref{equ_discretized_mLASSO}, we replace the neighborhood estimate $\widehat{\mathcal{N}}(r)$ given by \eqref{equ_def_est_neighborhood_LASSO_CS_sel} in Algorithm \ref{alg:1} with 
\begin{equation}
\label{equ_def_est_neighborhood_LASSO_CS_sel_discr}
\widehat{\mathcal{N}}(r) = \{\Pi_{r}(r'+1)  \, | \, r' \in [p] \mbox{, }   (1/\sqrt{\nrtasks}) \| \hat{\vbe}^{(r')} \|_{2} > \eta\}, 
\end{equation} 
where $\hat{\vbe}^{(r)} \defeq  \big(  \big( \hat{\vbe}_{1} \big)_{r},\ldots,\big( \hat{\vbe}_{\nrtasks} \big)_{r} \big)$.

{
We note that Algorithm \ref{alg:1}, based on the discretized version \eqref{equ_discretized_mLASSO} of the mLASSO \eqref{equ_def_multitask_LASSO}, scales well with the problem dimensions, i.e., it can be 
implemented efficiently for large sample size $\samplesize$ 
and large number $\coefflen$ of process components. Indeed, 
the expressions \eqref{equ_explicit_extression_y_DFT_tilde} and \eqref{equ_explicit_extression_X_DFT_tilde} can be evaluated efficiently using FFT algorithms. For a fast implementation of the mLASSO \eqref{equ_discretized_mLASSO} we refer to \cite{QinGoldfarb}. 
}

\section{Selection Consistency of the Proposed Scheme} 
\label{sec_var_sel_consist} 

We will now analyze the probability of Algorithm \ref{alg:1} to deliver a wrong CIG. {Our approach is to separately bound the probability that a specific neighborhood $\mathcal{N}(r)$, for an arbitrary but fixed node $r\!\in\![\coefflen]$, is estimated incorrectly by Algorithm \ref{alg:1}. Since the correct determination of all neighborhoods implies the delivery of the correct CIG, we can invoke a union bound over all $\coefflen$ neighborhoods to finally obtain an upper bound on the error probability of the GMS method. For clarity, we detail the analysis only for the specific neighborhood $\mathcal{N}(1)$, the generalization to an arbitrary neighborhood $\mathcal{N}(r)$ being trivially obtained by considering the permuted process $\tilde{\vx}[\timeidx] = \mathbf{P}_{r} \vx[\timeidx]$ (see our discussion around \eqref{equ_relation_CIG_permuted}). 
}

{
The high-level idea is to divide the analysis into a deterministic part and a stochastic part. 
The deterministic part consists of a set of sufficient conditions on the multitask learning problem \eqref{equ_def_signal_model_given_cov_matrix} such that the generalized support of the mLASSO $\hat{\vbe}[\mathbf{y}(\cdot),\mathbf{X}(\cdot)]$ (cf. \eqref{equ_def_multitask_LASSO}), coincides with the generalized support of the parameter vector $\vbe(\theta)$ in \eqref{equ_def_parameter_vector_given_cov_matrix}, which, in turn, is equal to $\mathcal{N}(1)\!-\!1$ (cf.\ \eqref{eq:suppeqbeta}). 
These conditions are stated in Theorem \ref{them_support_selection_consistency_multitask_LASSO} below. The stochastic part of the analysis amounts to controlling the probability that the 
sufficient conditions of Theorem \ref{them_support_selection_consistency_multitask_LASSO} are satisfied. This will be accomplished by a large deviation analysis of the BT estimator in \eqref{equ_est_BT_sdm_theta_t}. By combining these two parts, we straightforwardly obtain our main result, i.e., Theorem \ref{thm_main_result_performance_guarantee} which presents a condition on the sample size $\samplesize$ such that the error probability of our GMS method is upper bounded by a prescribed value.}

{\emph{Deterministic Part.}} The deterministic part of our analysis is based on the concept of the multitask compatibility condition \cite{BuhlGeerBook}. For a given index set $\mathcal{S} \subseteq [q]$ of size $\sparsity$,  the system matrix $\mathbf{X}(\theta) \in \mathbb{C}^{(2\samplesize\!-\!1) \times (\coefflen\!-\!1)}$, defined for $\theta \in [0,1)$, is said to satisfy the multitask compatibility condition 
with constant $\phi(\S)$ if 
\vspace*{0mm}
\begin{equation} 
\label{equ_def_multitask_compatibility_condition_bound}
\sparsity  \frac{\norm[2]{ \mX(\cdot) \vbed(\cdot) }^{2} }{\| \vbed_{\S}(\cdot) \|_{1}^{2}}  \geq \phi^{2}(\S) > 0 
\vspace*{0mm}
\end{equation} 
for all vectors $\vbed(\cdot) \in \mathbb{A}(\S) \setminus \{\mathbf{0}\}${, where 
\begin{equation}
\label{equ_norm_inequality_2_1_norm_compatibility}
\mathbb{A}(\S) \! \defeq \! \big\{  \vbed(\cdot)\!\in\!\ell_{q}([0,1))   \big| \| \vbed_{\comp{\S}}(\cdot) \|_{1}\!\leq\!3 \| \vbed_{\S}(\cdot) \|_{1} \}.
\end{equation}} 
Another quantity which is particularly relevant for the variable selection performance of the mLASSO is the minimum norm $\min_{r \in \gsupp({\vbe(\cdot)})} \| \beta_{r}(\cdot) \|_{L^{2}}$ of the non-zero blocks of the parameter vector $\vbe(\cdot) \in \ell_{q}([0,1))$ given by \eqref{equ_def_parameter_vector_given_cov_matrix}. {We require this quantity to be lower bounded by a known positive number $\beta_{\text{min}}$, i.e., 
\begin{equation}
\label{equ_beta_min_multitask_problem}
 \min_{r \in \gsupp({\vbe(\cdot)})} \| \beta_{r}(\cdot) \|_{L^{2}} \geq \beta_{\text{min}}.
\end{equation}}

Based on $\phi(\S)$ and $\beta_{\text{min}}$, the following result characterizes the ability of the mLASSO $\hat{\vbe}[\mathbf{y}(\cdot),\mathbf{X}(\cdot)]$ (cf. \eqref{equ_def_multitask_LASSO}) to correctly identify the generalized support $\gsupp(\vbed(\cdot))$, which is equal to $\mathcal{N}(1)\!-\!1$ (cf.\ \eqref{eq:suppeqbeta}).
\begin{theorem} 
\label{them_support_selection_consistency_multitask_LASSO}
Consider the multitask learning model \eqref{equ_def_signal_model_given_cov_matrix} with parameter vector ${\bm \beta}(\cdot) \in \ell_{q}([0,1))$ and system matrix $\mX(\theta)$. 
The parameter vector $\vbed(\cdot)$ is assumed to have no more than $\sparsity$ non-zero components, i.e., 
\begin{equation}
\gsupp({\bm \beta}(\cdot)) \subseteq \S \mbox{, with } |\S| = \sparsity.  
\end{equation}  
Assume further that the system matrix possesses a positive multitask compatibility constant $\phi(\S)>0$ (cf.\ \eqref{equ_def_multitask_compatibility_condition_bound}), and the 
error term $\vep(\theta)$ in \eqref{equ_def_signal_model_given_cov_matrix} satisfies 
\begin{equation} 
\label{equ_condition_corr_error_corr_support_selection1}
\sup_{\theta \in [0,1)} \norm[\infty]{\herm{\vep}(\theta) \mX(\theta) } \leq \frac{\phi^{2}(\S) \beta_{\emph{min}}}{32\sparsity}. 
\end{equation}
Denote by $\hat{\vbe}[\mathbf{y}(\cdot),\mathbf{X}(\cdot)]$ the mLASSO estimate obtained from \eqref{equ_def_multitask_LASSO} with $\lambda = \phi^{2}(\S) \beta_{\emph{min}}/(8 \sparsity)$. 
Then, the index set 
\begin{align}
\label{eq:corcondeta_success}
\hat \S \defeq \big\{r \in [q] \,\,  | \,\, \| \hat{\beta}_{r}(\cdot) \|_{L^{2}} > \beta_{\emph{min}}/2 \big\},  
\end{align}
coincides with the true generalized support of $\vbe(\cdot)$, i.e., $\hat \S =\gsupp({\vbe}(\cdot))$.
\end{theorem} 
\begin{proof}
Appendix \ref{app_proof_them_support_selection_consistency_multitask_LASSO}.
\end{proof}

{\emph{Stochastic Part.}} 
{We now show that,} for sufficiently large sample size $\samplesize$, the multitask learning problem \eqref{equ_def_signal_model_given_cov_matrix} satisfies the condition \eqref{equ_condition_corr_error_corr_support_selection1} of Theorem \ref{them_support_selection_consistency_multitask_LASSO} with high probability. 
To this end, we first verify that \eqref{equ_condition_corr_error_corr_support_selection1} is satisfied if the maximum SDM estimation error 
\vspace*{-2mm}
\begin{equation}
\label{equ_def_error_EST_E}
E \defeq \sup_{\theta \in [0,1)} \norm[\infty]{\mathbf{E}(\theta)} \mbox{, with } \mathbf{E}(\theta) \defeq \ESDM(\theta) - \SDM(\theta), 
\vspace*{-1mm}
\end{equation}
is small enough. We then characterize the large deviation behavior of $E$ to obtain an upper bound on the probability of Algorithm \ref{alg:1} to deliver a wrong neighborhood, i.e., 
we bound the probability $\prob \{\widehat{\mathcal{N}}(r) \neq \mathcal{N}(r) \}$, {for an arbitrary but fixed node $r\in[\coefflen]$.} 

In order to invoke Theorem \ref{them_support_selection_consistency_multitask_LASSO}, we need to ensure $\beta_{\text{min}} = \min\limits_{r \in \gsupp({\bm \beta}(\cdot))} \| \beta_{r}(\cdot) \|_{L^{2}}$ (with ${\bm \beta}(\cdot)$ given by \eqref{equ_def_parameter_vector_given_cov_matrix}) to be sufficiently large. This is accomplished by assuming \eqref{equ_condition_beta_min_rho}, which is valid for any process $\vx[\timeidx] \in \mathcal{M}$, and implying via \eqref{equ_matrix_inversion_block_mtx_parameter_vector} the lower bound 
\begin{equation}
\label{equ_lower_bound_rho_min_tilde_beta_min}
\beta_{\text{min}} \geq \rho_{\text{min}}.
\end{equation}

In order to ensure validity of \eqref{equ_condition_corr_error_corr_support_selection1}, we need the following relation between the maximum correlation $\sup_{\theta \in [0,1)} \norm[\infty]{\herm{\vep}(\theta) \mX(\theta) }$ and the estimation error $E$ in \eqref{equ_def_error_EST_E}. 
\begin{lemma}
\label{lem_relation_corr_hat_SDM_err}
Consider the multitask learning problem \eqref{equ_def_signal_model_given_cov_matrix}, {with observation vector $\mathbf{y}(\theta)$ and system matrix $\mX(\theta)$ given 
by \eqref{equ_explicit_extression_y_DFT_tilde} and \eqref{equ_explicit_extression_X_DFT_tilde}, based on the permuted observation $\widetilde{\mathbf{D}}=\mathbf{P}_{r} \big(\vx[1],\ldots,\vx[\samplesize]\big)$ of the process $\mathbf{x}[\timeidx] \in \mathcal{M}$}. We have 
\begin{equation}
\label{equ_max_corr_maximum_SDM_estimation_error}
\sup_{\theta \in [0,1)} \norm[\infty]{\herm{\vep}(\theta) \mX(\theta) } \leq 2 E \sqrt{s_{\emph{max}}} U.
\end{equation} 
\end{lemma} 
\begin{proof}
Appendix \ref{app_proof_lem_relation_corr_hat_SDM_err}.
\end{proof} 

Note that due to \eqref{equ_max_corr_maximum_SDM_estimation_error} and \eqref{equ_lower_bound_rho_min_tilde_beta_min}, a sufficient condition for \eqref{equ_condition_corr_error_corr_support_selection1} to be satisfied is 
\begin{equation}
\label{equ_suff_cond_max_corr_SDM_est}
E \leq \phi^{2}(\S) \rho_{\text{min}} / (64 U \sparsity^{3/2}).
\end{equation} 

The following result characterizes the multitask compatibility condition $\phi(\mathcal{S})$ of the system matrix $\mathbf{X}(\theta)$ given by \eqref{equ_explicit_extression_X_DFT}, for a process $\mathbf{x}[\timeidx]$ belonging to $\mathcal{M}$, i.e., in particular satisfying \eqref{equ_uniform_bound_eigvals_specdensmatrix}.
\begin{lemma}
\label{lem_compat_hat_multitask_problem_satisfied}
Consider the multitask learning problem \eqref{equ_def_signal_model_given_cov_matrix} which is constructed according to \eqref{equ_explicit_extression_y_DFT}, \eqref{equ_explicit_extression_X_DFT}, based on the observed process $\mathbf{x}[\timeidx] \in \mathcal{M}$.
If the estimation error $E$ in \eqref{equ_def_error_EST_E} satisfies 
\begin{equation}
\label{equ_condition_max_diff_infty_norm_SDM_task}
E \leq 1/(32 \sparsity)  ,
\end{equation}
then, for any subset $\mathcal{S} \subseteq [\coefflen]$ with $|\mathcal{S}| \leq \sparsity$, the system matrix $\mX(\theta)${, given for any $\theta \!\in\![0,1)$ by \eqref{equ_explicit_extression_X_DFT},} satisfies the multitask compatibility condition \eqref{equ_def_multitask_compatibility_condition_bound} with a constant 
\begin{equation}
\label{equ_lower_bound_phi_tilde}
\phi(\mathcal{S}) \geq 1/\sqrt{2}.
\end{equation}
\end{lemma}
\begin{proof}
Appendix \ref{app_proof_lem_compat_hat_multitask_problem_satisfied}.
\end{proof} 
Combining Lemma \ref{lem_compat_hat_multitask_problem_satisfied} with the sufficient condition \eqref{equ_suff_cond_max_corr_SDM_est}, we have that the multitask learning problem \eqref{equ_def_signal_model_given_cov_matrix} 
satisfies the requirement \eqref{equ_condition_corr_error_corr_support_selection1} of Theorem \ref{them_support_selection_consistency_multitask_LASSO} if 
\begin{align}
\label{equ_suff_cond_error_compat_constant_correlation_simul}
E & \leq  \frac{\rho_{\text{min}}  }{128 U \sparsity^{3/2}}.
\end{align} 
Indeed, the validity of \eqref{equ_suff_cond_error_compat_constant_correlation_simul} implies \eqref{equ_condition_max_diff_infty_norm_SDM_task} since $\rho_{\text{min}} \leq U$ which can be verified from the assumption \eqref{equ_condition_beta_min_rho} and the relations $\big| \big(\SDM^{-1}(\theta) \big)_{r,r} \big| \geq \lambda_{\text{min}} \big( \SDM^{-1}(\theta) \big) \stackrel{\eqref{equ_uniform_bound_eigvals_specdensmatrix}}\geq 1/U$, and  $\big| \big(\SDM^{-1}(\theta) \big)_{r,r'} \big| \leq \lambda_{\text{max}} \big( \SDM^{-1}(\theta) \big) \stackrel{\eqref{equ_uniform_bound_eigvals_specdensmatrix}}\leq 1$

In what follows, we derive an upper bound on the probability that \eqref{equ_suff_cond_error_compat_constant_correlation_simul} is not satisfied for a process $\mathbf{x}[\timeidx] \in \mathcal{M}$. This will be done with the aid of 
\newcommand\m{\mu} 
\begin{lemma} 
\label{lem_large_deviations_SDM_est}
Let $\ESDM(\theta)$ be the estimate of $\SDM(\theta)$, obtained according to \eqref{equ_est_BT_sdm_theta_t} with sample size $\samplesize$ and window function $w[\cdot] \in \ell_{1}(\mathbb{Z})$. 
For $\nu \in [0,1/2)$, 
\begin{align}
\PR\{ E \geq \nu + \mu_{x}^{(h_{1})} \} \leq 2e^{-\frac{N \nu^2}{8 \| w[\cdot] \|^{2}_{1} U^2}  +2\log p + \log 2N }.
\label{eq:boundesterr}
\end{align}
where $\mu_{x}^{(h_{1})}$ denotes the ACF moment \eqref{equ_def_generic_moments_ACF} obtained for the weighting function 
\begin{equation}
\label{equ_def_weight_func_lem_large_deviations_SDM_est}
h_{1}[\lagvar] \defeq \begin{cases} \left| 1- w[\lagvar](1-|\lagvar|/ \samplesize) \right|  & \mbox{ for } |\lagvar| < \samplesize  \\ \hspace*{14mm} 1 & \mbox{ else.} \end{cases}
\end{equation} 
\end{lemma}
\begin{proof}
Appendix \ref{app_proof_lem_large_deviations_SDM_est}. 
\end{proof} 
{\emph{Main Result.}}
{By Lemma \ref{lem_large_deviations_SDM_est}, we can characterize the probability of the condition \eqref{equ_suff_cond_error_compat_constant_correlation_simul} to hold. Since validity of \eqref{equ_suff_cond_error_compat_constant_correlation_simul} allows to invoke Theorem \ref{them_support_selection_consistency_multitask_LASSO}, we arrive at }
\begin{theorem} 
\label{thm_main_result_performance_guarantee}
Consider a process $\mathbf{x}[\timeidx] \in \mathcal{M}$ and the corresponding SDM estimate \eqref{equ_est_BT_sdm_theta_t}. 
Then, if 
\begin{align}
\label{equ_condition_sample_size_succ_recovery}
\frac{\samplesize  (\rho_{\emph{min}}/256)^2}{8 \sparsity^3 \| w[\cdot] \|^{2}_{1} U^4} - \log(2 \samplesize)  & \geq \log(2\coefflen^{2}/\delta) \mbox{, and }  \\
\mu_{x}^{(h_{1})}& \leq  \frac{\rho_{\emph{min}}  }{256 U \sparsity^{3/2}}, \label{equ_condition_sample_size_succ_recovery_2}
\end{align}
the probability of Algorithm \ref{alg:1}, using $\lambda \! =\! \rho_{\emph{min}}/(16 \sparsity)$ in \eqref{equ_def_multitask_LASSO} and $\eta\!=\!\rho_{\emph{min}}/2$ in \eqref{equ_def_est_neighborhood_LASSO_CS_sel}, selecting the neighborhood of node $r \in [p]$ not correctly, i.e., 
$\widehat{\mathcal{N}}(r) \neq \mathcal{N}(r)$, is upper bounded as $\prob \{ \widehat{\mathcal{N}}(r) \neq \mathcal{N}(r) \} \leq \delta$.
\end{theorem} 

{Note that Theorem \ref{thm_main_result_performance_guarantee} applies to the infinite dimensional mLASSO optimization problem in \eqref{equ_def_multitask_LASSO}, thereby ignoring any discretization or numerical implementation issue. Nevertheless, if the discretization is fine enough, i.e., the number $\nrtasks$ of frequency intervals used for the discretized mLASSO \eqref{equ_discretized_mLASSO} is sufficiently large,  
we expect that Theorem \ref{thm_main_result_performance_guarantee} accurately predicts the performance of the GMS method obtained by using Algorithm \ref{alg:1} with 
the discretized mLASSO \eqref{equ_discretized_mLASSO} instead of the infinite dimensional mLASSO \eqref{equ_def_multitask_LASSO}.} 

{
Furthermore, Theorem \ref{thm_main_result_performance_guarantee} considers the probability of (the first two steps of) Algorithm \ref{alg:1} to fail in selecting the correct neighborhood $\mathcal{N}(r)$ of a specific node $r$. 
Since any reasonable combination strategy in step $3$ of Algorithm \ref{alg:1} will yield the correct CIG if all neighborhoods are estimated correctly, we obtain, via a union bound over all nodes $r \in [\coefflen]$, the following bound on the probability that $p$ applications of Algorithm 1 (one for each node) 
yields a wrong CIG.}
\begin{corollary} 
\label{cor_main_result_performance_guarantee}
Consider a process $\mathbf{x}[\timeidx] \in \mathcal{M}$ and the corresponding SDM estimate \eqref{equ_est_BT_sdm_theta_t}. 
Then, if 
\begin{align}
\label{equ_condition_sample_size_succ_recovery}
\frac{\samplesize (\rho_{\emph{min}}/256)^2 }{8 \sparsity^3 \| w[\cdot] \|^{2}_{1} U^4} - \log(2 \samplesize)  & \geq \log(2\coefflen^{3}/\delta) \mbox{, and }  \\
\mu_{x}^{(h_{1})}& \leq  \frac{\rho_{\emph{min}}  }{256 U \sparsity^{3/2}}, \label{equ_condition_sample_size_succ_recovery_2}
\end{align}
the probability of Algorithm \ref{alg:1}, applied sequentially to all nodes $r \in[p]$, using $\lambda \! =\! \rho_{\emph{min}}/(16 \sparsity)$ in \eqref{equ_def_multitask_LASSO} and $\eta\!=\!\rho_{\emph{min}}/2$ in \eqref{equ_def_est_neighborhood_LASSO_CS_sel}, yielding a wrong CIG, i.e., $\widehat{\mathcal{G}} \! \neq \! \mathcal{G}$, is upper bounded as $\prob \{ \widehat{\mathcal{G}}_{x} \! \neq \! \mathcal{G}_{x} \} \leq \delta$.
\end{corollary} 

According to \eqref{equ_condition_sample_size_succ_recovery}, neglecting the term $\log(2\samplesize)$ and assuming $\rho_{\text{min}}$ fixed, the sample size $\samplesize$ has to grow polynomially with the maximum node degree and logarithmically with the number of process components $\coefflen$. 
{This polynomial and logarithmic scaling of the sample size $\samplesize$ on the maximum node degree $\sparsity$ and number of process components $\coefflen$, respectively, is a typical requirement for accurate GMS in the high-dimensional regime \cite{RavWainLaff2010,MeinBuhl2006,RavWainRaskYu2011}.}

Note that, according to \eqref{equ_condition_sample_size_succ_recovery}, the sample size $\samplesize$ has to grow with the squared $\ell_{1}$ norm $\| w[\cdot] \|^{2}_{1}$ of the window function $w[\cdot]$ employed in the BT estimator \eqref{equ_est_BT_sdm_theta_t}. 
For the inequality \eqref{equ_condition_sample_size_succ_recovery_2} to hold, one typically has to use a window function $w[\cdot]$ whose effective support matches those of the process ACF $\ACF[\lagvar]$. Therefore, 
Theorem \ref{thm_main_result_performance_guarantee} suggests that the sample size has to grow with the square of the effective process correlation width (effective size of the ACF support), which is quantified by $\mu_{x}$. 
However, some first results on the fundamental limits of GMS for time series in indicate that the required sample size should be effectively independent of the correlation width $\mu_{x}$ \cite{HannakJung2014conf}. 

One explanation of the discrepancy between the sufficient condition \eqref{equ_condition_sample_size_succ_recovery} and the lower bounds on the required sample size is that the derivation of Theorem \ref{thm_main_result_performance_guarantee} is based on requiring the SDM estimator $\ESDM(\theta)$, given by \eqref{equ_est_BT_sdm_theta_t}, to be accurate \emph{simultaneously} for all $\theta \! \in \! [0,1)$. According to \cite{CaiRenZhou}, the achievable uniform estimation accuracy, measured by the minimax risk, depends inversely on the correlation width $\mu_{x}$. However, the analysis in \cite{HannakJung2014conf} suggests that it is not necessary to accurately estimate the SDM $\SDM(\theta)$ for all $\theta$ simultaneously. {Indeed, for a process $\vx[\timeidx]$ with underlying CIG $\cig$, the SDM values $\SDM(\theta)$ are coupled over frequency $\theta \in [0,1)$ via the relation \eqref{equ_charact_edge_set_covariance_matrix_k}.} 
Due to this coupling, the SDM needs to be estimated accurately only on average (over frequency $\theta$). A more detailed performance analysis of the selection scheme in Algorithm \ref{alg:1}, taking the coupling effect due to \eqref{equ_charact_edge_set_covariance_matrix_k} into account, is an interesting direction 
for future work. 

\section{Numerical Experiments}
\label{sec_numerical_experiments}

{
The performance of the GMS method given by Algorithm \ref{alg:1} is assessed by two complementary numerical experiments. In the first experiment we 
measure the ability of our method to correctly identify the edge set of the CIG of a synthetically generated process. In a second experiment, we apply our 
GMS method to electroencephalography (EEG) measurements, demonstrating that the resulting CIG estimate may be used for detecting the eye state (open/closed) of a person.} 

\subsection{Synthetic Process} 

We generated a Gaussian process $\mathbf{x}[n]$ of dimension $p=64$ by applying a finite impulse 
response filter $g[\lagvar]$ of length $2$ to a zero-mean stationary white Gaussian noise process $\mathbf{e}[n] \sim \mathcal{N}(\mathbf{0},\mathbf{C}_{0})$. 
The covariance matrix $\mathbf{C}_{0}$ was chosen such that the resulting CIG $\cig=([p],\edges)$ satisfies \eqref{equ_def_maximum_node_degree} with $s_{\text{max}}=3$.
The non-zero filter coefficients $g[0]$ and $g[1]$ are chosen such that the magnitude of the associated transfer function is uniformly bounded from above and below by positive constants, thereby ensuring condition \eqref{equ_uniform_bound_eigvals_specdensmatrix}. 

{We then computed the CIG estimate $\widehat{\mathcal{G}}_{x}$ using Algorithm \ref{alg:1} based on the discretized version \eqref{equ_discretized_mLASSO} of the mLASSO (with $\nrtasks \!=\! 4$) and the window function $w[\lagvar] = \exp(- \lagvar^2 / 44 )$.} In particular, we applied 
the \emph{alternating direction method of multipliers (ADMM)} to the optimization problem \eqref{equ_discretized_mLASSO} (cf.\ \cite[Sec. 6.4]{DistrOptStatistLearningADMM}).\footnote{{We used the all-zero initialization for the ADMM variables in our experiments. In general, the convergence of the ADMM implementation for LASSO type problems of the form \eqref{equ_discretized_mLASSO} is not sensitive to the precise initialization of the optimization variables \cite{HongLuo2013,DistrOptStatistLearningADMM}.}}
We set $\lambda = c_{1} \phi^2_{\text{min}} \rho_{\text{min}} / (18 s_{\text{max}} \nrtasks)$ and $\eta = \rho_{\text{min}}/2$, where $c_{1}$ was varied in the range $c_{1} \in [10^{-3},10^{3}]$. 

In Fig.\ \ref{fig_ROC}, we show receiver operating characteristic (ROC) curves with the average fraction of false alarms {$P_{fa} \defeq \frac{1}{M} \sum_{l \in [M]} \frac{|\widehat{\edges}_{l} \setminus \edges|}{\coefflen(\coefflen\!-\!1)/2\!-\!|\edges|}$} and the average fraction of correct decisions {$P_{d} \defeq \frac{1}{M} \sum_{l \in [M]} \frac{|\widehat{\edges}_{l}|}{|\edges|}$} for varying mLASSO parameter $\lambda$. Here, $\widehat{\mathcal{E}}_{l}$ denotes the edge set estimate obtained from Algorithm \ref{alg:1} during the $l$-th simulation run. We averaged over $M=10$ i.i.d.\ simulation runs. 
As can be seen from Fig.\ \ref{fig_ROC}, our selection scheme yields reasonable performance even if $\samplesize\!=\!32$ only for a $64$-dimensional process. We also adapted an existing VAR-based network learning method \cite{Nowak2011} in order to 
estimate the underlying CIG. The resulting ROC curves are also shown in Fig.\ \ref{fig_ROC}. Note that the performance obtained for the VAR-based method is similar to a pure guess. The inferior performance of the VAR-based method is due to 
a model mismatch since the simulated process is not a VAR process {but a moving average process}.
\begin{figure}
\vspace{-1mm}
\centering
\psfrag{FA}[c][c][.85]{\uput{4mm}[270]{0}{\hspace{0mm}$P_{fa}$}}
\psfrag{ROC}[c][c][.9]{\uput{2.5mm}[90]{0}{}}
\psfrag{data1}[c][c][.9]{\hspace*{7mm} $N\!=\!128$}
\psfrag{data2}[c][c][.9]{\hspace*{5mm} $N\!=\!64$}
\psfrag{data3}[c][c][.9]{\hspace*{5mm} $N\!=\!32$}
\psfrag{data4}[c][c][.9]{\hspace*{12mm} $N\!=\!128$,VAR}
\psfrag{data5}[c][c][.9]{\hspace*{10mm} $N\!=\!64$,VAR}
\psfrag{data6}[c][c][.9]{\hspace*{10mm} $N\!=\!32$,VAR}
\psfrag{x_0}[c][c][.9]{\uput{0.3mm}[270]{0}{$0$}}
\psfrag{x_0_1}[c][c][.9]{\uput{0.3mm}[270]{0}{$0.1$}}
\psfrag{x_0_2}[c][c][.9]{\uput{0.3mm}[270]{0}{$0.2$}}
\psfrag{x_0_3}[c][c][.9]{\uput{0.3mm}[270]{0}{$0.3$}}
\psfrag{x_0_4}[c][c][.9]{\uput{0.3mm}[270]{0}{$0.4$}}
\psfrag{x_0_5}[c][c][.9]{\uput{0.3mm}[270]{0}{$0.5$}}
\psfrag{x_0_6}[c][c][.9]{\uput{0.3mm}[270]{0}{$0.6$}}
\psfrag{x_0_7}[c][c][.9]{\uput{0.3mm}[270]{0}{$0.7$}}
\psfrag{x_0_8}[c][c][.9]{\uput{0.3mm}[270]{0}{$0.8$}}
\psfrag{x_0_9}[c][c][.9]{\uput{0.3mm}[270]{0}{$0.9$}}
\psfrag{x_1}[c][c][.9]{\uput{0.3mm}[270]{0}{$1$}}
\psfrag{y_0}[c][c][.9]{\uput{0.3mm}[180]{0}{$0$}}
\psfrag{y_0_1}[c][c][.9]{\uput{0.3mm}[180]{0}{$0.1$}}
\psfrag{y_0_2}[c][c][.9]{\uput{0.3mm}[180]{0}{$0.2$}}
\psfrag{y_0_3}[c][c][.9]{\uput{0.3mm}[180]{0}{$0.3$}}
\psfrag{y_0_4}[c][c][.9]{\uput{0.3mm}[180]{0}{$0.4$}}
\psfrag{y_0_5}[c][c][.9]{\uput{0.3mm}[180]{0}{$0.5$}}
\psfrag{y_0_6}[c][c][.9]{\uput{0.3mm}[180]{0}{$0.6$}}
\psfrag{y_0_7}[c][c][.9]{\uput{0.3mm}[180]{0}{$0.7$}}
\psfrag{y_0_8}[c][c][.9]{\uput{0.3mm}[180]{0}{$0.8$}}
\psfrag{y_0_9}[c][c][.9]{\uput{0.3mm}[180]{0}{$0.9$}}
\psfrag{y_1}[c][c][.9]{\uput{0.3mm}[180]{0}{$1$}}
\psfrag{PD}[c][c][.9]{\uput{4mm}[90]{0}{\hspace{0mm}$P_{d}$}}
\centering
\includegraphics[height=6cm,width=9.2cm]{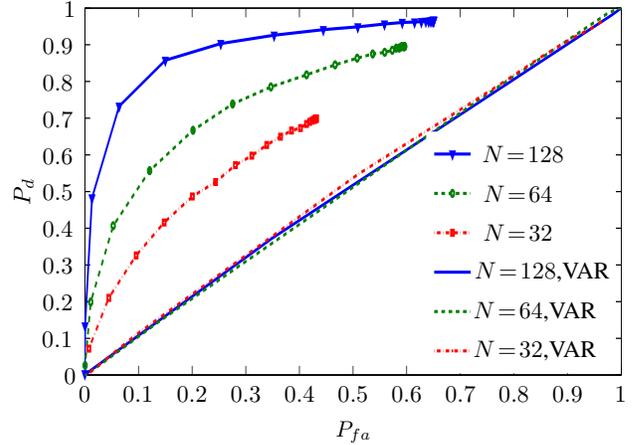}
\vspace{2mm}
  \caption{ROC curves for the compressive selection scheme given by Algorithm \ref{alg:1} and for a VAR-model based GMS scheme presented in \cite{Nowak2011}.} 
\label{fig_ROC}
\vspace*{-1mm}
\end{figure}

We also evaluated the empirical detection probability $P_{d}$ for fixed mLASSO parameter $\lambda=\rho_{\text{min}}/10$ and varying rescaled sample size $\tau \! \defeq \! \samplesize / ( \log(\coefflen) s_{\text{max}}^{3})$. According to Fig.\ \ref{fig_P_d_vs_rescaled_sample_size}, and as suggested by the bound \eqref{equ_condition_sample_size_succ_recovery} of Theorem \ref{thm_main_result_performance_guarantee}, for a fixed squared norm $\|w[\cdot]\|_{1}^{2}$ (the window function $w[\lagvar]$ employed in \eqref{equ_est_BT_sdm_theta_t} is fixed throughout the simulation), the rescaled sample size $\tau=\samplesize / ( \log(\coefflen) s_{\text{max}}^{3})$ seems to be an accurate performance indicator. In particular, the selection scheme in Algorithm \ref{alg:1} works well only if $\tau \gg 1$. 

\vspace*{2mm} 
\begin{figure}
\centering
\psfrag{xlabel}[c][c][.85]{\uput{3.8mm}[270]{0}{\hspace{0mm}$\log(\tau)$}}
\psfrag{title}[c][c][.9]{\uput{1mm}[90]{0}{}}
\psfrag{x_0}[c][c][.9]{\uput{0.3mm}[270]{0}{$0$}}
\psfrag{x_m_1}[c][c][.9]{\uput{0.3mm}[270]{0}{$-1$}}
\psfrag{x_m_0_5}[c][c][.9]{\uput{0.3mm}[270]{0}{$-0.5$}}
\psfrag{x_m_2}[c][c][.9]{\uput{0.3mm}[270]{0}{$-2$}}
\psfrag{x_m_3}[c][c][.9]{\uput{0.3mm}[270]{0}{$-3$}}
\psfrag{x_m_4}[c][c][.9]{\uput{0.3mm}[270]{0}{$-4$}}
\psfrag{x_0_5}[c][c][.9]{\uput{0.3mm}[270]{0}{$0.5$}}
\psfrag{x_1}[c][c][.9]{\uput{0.3mm}[270]{0}{$1$}}
\psfrag{x_2}[c][c][.9]{\uput{0.3mm}[270]{0}{$2$}}
\psfrag{x_3}[c][c][.9]{\uput{0.3mm}[270]{0}{$3$}}
\psfrag{y_0}[c][c][.9]{\uput{0.3mm}[180]{0}{$0$}}
\psfrag{data1}[c][c][.9]{\hspace*{6mm}$p=64$}
\psfrag{data2}[c][c][.9]{\hspace*{6mm}$p=128$}
\psfrag{data3}[c][c][.9]{\hspace*{6mm}$p=160$}
\psfrag{y_0_1}[c][c][.9]{\uput{0.3mm}[180]{0}{$0.1$}}
\psfrag{y_0_2}[c][c][.9]{\uput{0.3mm}[180]{0}{$0.2$}}
\psfrag{y_0_3}[c][c][.9]{\uput{0.3mm}[180]{0}{$0.3$}}
\psfrag{y_0_4}[c][c][.9]{\uput{0.3mm}[180]{0}{$0.4$}}
\psfrag{y_0_5}[c][c][.9]{\uput{0.3mm}[180]{0}{$0.5$}}
\psfrag{y_0_6}[c][c][.9]{\uput{0.3mm}[180]{0}{$0.6$}}
\psfrag{y_0_7}[c][c][.9]{\uput{0.3mm}[180]{0}{$0.7$}}
\psfrag{y_0_8}[c][c][.9]{\uput{0.3mm}[180]{0}{$0.8$}}
\psfrag{y_0_9}[c][c][.9]{\uput{0.3mm}[180]{0}{$0.9$}}
\psfrag{y_1}[c][c][.9]{\uput{0.3mm}[180]{0}{$1$}}
\psfrag{ylabel}[c][c][.9]{\uput{6mm}[90]{0}{\hspace{0mm}$P_{d}$}}
\centering
\includegraphics[height=6cm,width=9.2cm]{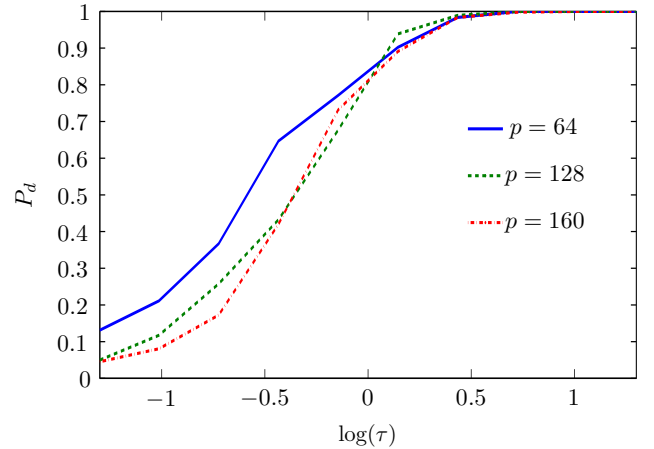}
\caption{Empirical detection probability $P_{d}$ vs. rescaled sample size $\tau = \samplesize/(\log(\coefflen)s_{\text{max}}^{3})$.}
\label{fig_P_d_vs_rescaled_sample_size}
\vspace*{3mm}
\end{figure}

{
\subsection{Eye State Detection} 
In this experiment, we evaluate the applicability of our GMS method for the problem of eye state detection based on EEG measurement data. 
This problem is relevant, e.g., for medical care or for driving drowsiness detection \cite{WangGuan2014}. We used the EEG dataset donated by 
Oliver Roesler from Baden-Wuerttemberg Cooperative State University (DHBW), Stuttgart, Germany, and available at the UCI machine learning repository \cite{BacheLichman2013}.
The dataset consists of 14980 time samples, each sample {being a vector made up of 14 feature values}. The true eye state was detected via a camera during the EEG recording. 
}

{
As a first processing step, given the raw data, we removed parts of the time series which contain outliers. In a second step we performed a detrending operation by applying a 
boxcar filter of length $5$. Based on the true eye state signal, which is equal to one if the eye was open and equal to zero if it was closed,  we extracted two data blocks $\mathbf{D}_{0}$, $\mathbf{D}_{1}$, one corresponding to each state. 
We then applied Algorithm \ref{alg:1} with the discretized mLASSO \eqref{equ_discretized_mLASSO} (with $\nrtasks=5$) instead of \eqref{equ_def_multitask_LASSO} and using the OR-rule in the third step, i.e., $\widehat{\mathcal{G}}$ contains the edge $(r,r')$ if either $r \in \widehat{\mathcal{N}}(r')$ or $r' \in \mathcal{N}(r)$. For the window function in the BT estimator \eqref{equ_est_BT_sdm_theta_t} we used the choice $w[\lagvar] = \exp(-(\lagvar/59)^2)$. 
In Fig. \ref{fig_eye_state}, we show the two CIG estimates obtained for each of the two data blocks $\mathbf{D}_{0},\mathbf{D}_{1} \in \mathbb{R}^{14 \times 1024}$ each corresponding to a sample size of $\samplesize=1024$. As evident from Fig. \ref{fig_eye_state}, 
the resulting graph estimates are significantly different for each of the two eye states. In particular, the graph obtained for the ``eye closed'' state contains much more edges 
which are moreover localized at few nodes having relatively high degree. Thus, the CIG estimate delivered by Algorithm \ref{alg:1} could serve as an indicator  
for the eye state of a person based on EEG measurements.}

\begin{figure}
\hspace*{-5mm}
\begin{minipage}{0.4\columnwidth} 
\psfrag{(a)}[c][c][.9]{\uput{0.5mm}[270]{0}{(a) ``eye open''}}
\includegraphics[width=5cm] {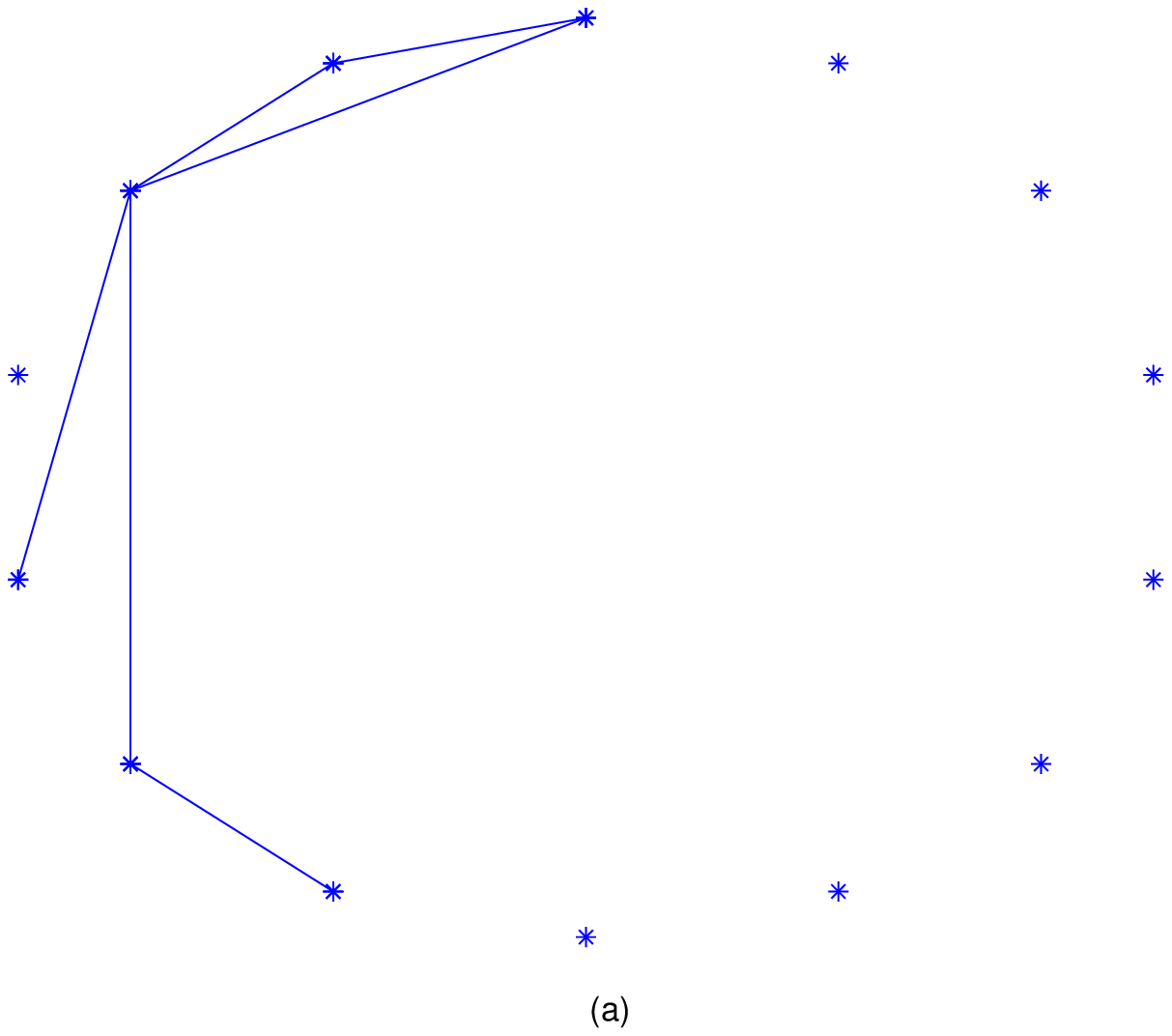}
	\end{minipage} \hspace*{8mm}
	\begin{minipage}{0.4\columnwidth} 
	\psfrag{(b)}[c][c][.9]{\uput{0.5mm}[270]{0}{(b) ``eye closed''}}
	  \includegraphics[width=5cm] {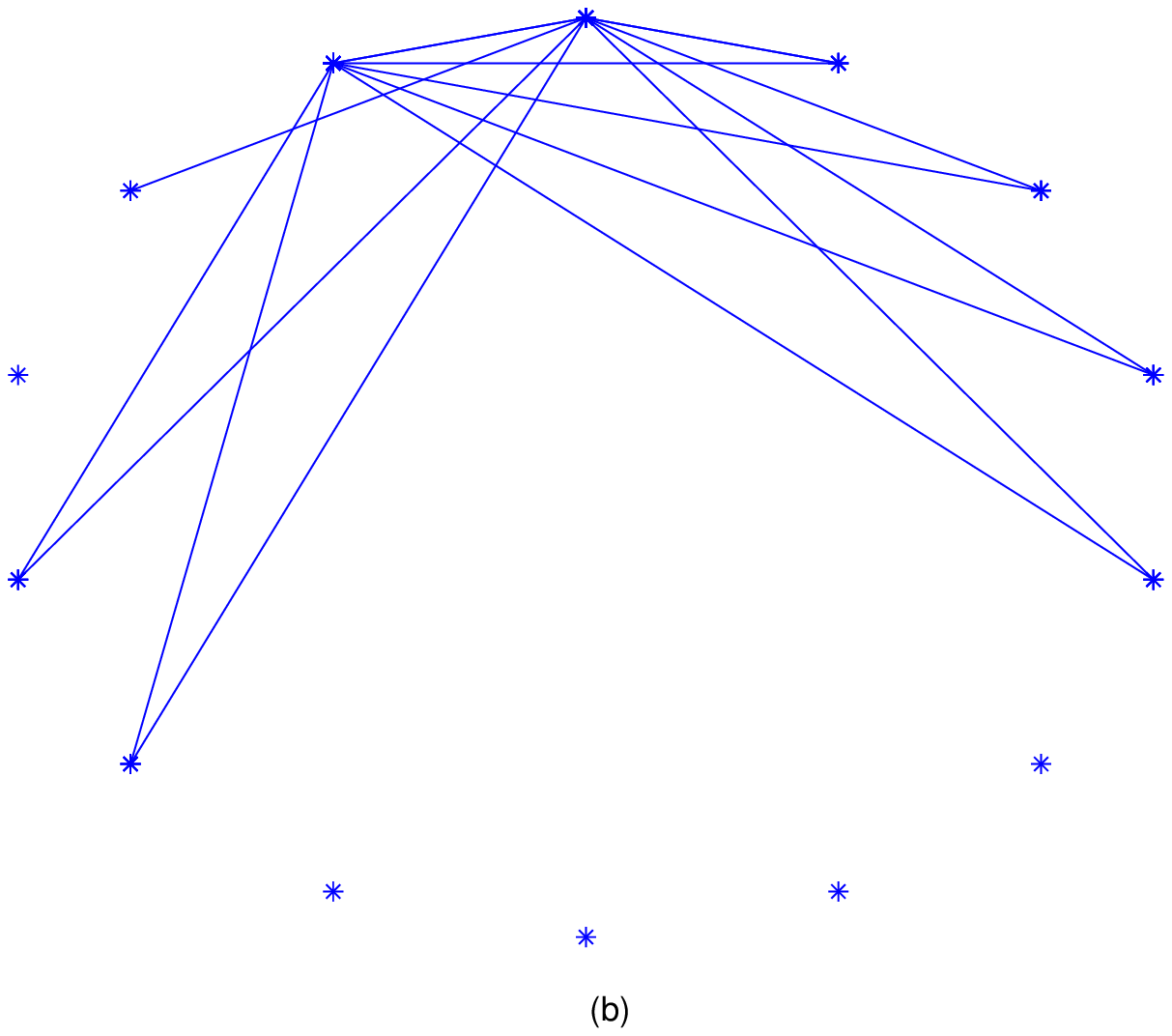}
	\end{minipage}
\vspace*{3mm}
\caption{Resulting CIG estimate for the EEG time series under different eye states.} \label{fig_eye_state}
\end{figure}

\section{Conclusion}

We proposed a nonparametric compressive selection scheme for inferring the CIG of a stationary discrete-time Gaussian vector process. This selection scheme is based on combining a BT estimator for the SDM with the mLASSO. The key idea behind this novel selection scheme is the formulation of 
the GMS problem for a stationary vector process as a multitask learning problem. This formulation lends itself to applying mLASSO to GMS for stationary vector processes. Drawing on an established performance characterization \cite{BuhlGeerBook} of the mLASSO, we derived sufficient conditions on the observed sample size such that the probability of selecting a wrong CIG does not exceed a given (small) threshold. Some numerical experiments validate our theoretical performance analysis and show superior performance compared to an existing (VAR-based) method in case of model mismatch. 

{Our work may serve as the basis for some interesting avenues of further research, e.g., extending the concept of a CIG to 
processes with a singular SDM or introducing the notion of a frequency dependent CIG. Moreover, we expect that our frequency domain approach to GMS for stationary vector processes can be extended easily to non-stationary vector processes by using time-frequency concepts (e.g., based on underspread assumptions).}

\section{Acknowledgment}

The author is grateful to R.\ Heckel who performed a careful review of some early manuscripts, thereby pointing to some errors 
in the consistency analysis and the formulation of Lemma \ref{thm_charact_error_multitask_LASSO}. 
Moreover, some helpful comments from and discussions with H.\ B{\"o}lcskei and F.\ Hlawatsch, resulting in an improved presentation of the main ideas, are appreciated sincerely.

\appendices 

\section{Proof of Lemma \ref{lem_repr_ESDM_factor}}
\label{lem_repr_ESDM_factor_proof}

Let $\tilde{x}_{r}[\timeidx]$ and $\tilde{x}_{t}[\timeidx]$ denote $(2\samplesize\!-\!1)$-periodic discrete-time signals, with one period given by  
\begin{equation}
\label{equ_def_periodic_2N_x_mn}
\tilde{x}_{\{r,t\}}[\timeidx-1] \defeq \begin{cases} (\mathbf{x}[\timeidx])_{\{r,t\}} & \mbox{ for }  \timeidx  \in [\samplesize],  \\Ê
0 & \mbox{ for }  \timeidx \in [2\samplesize\!-\!1] \setminus [\samplesize] 
\end{cases}
\end{equation} 
and corresponding DFTs 
\begin{align} 
\tilde{X}_{\{r,t\}}[k] &\!\defeq\! \sum_{\timeidx=0}^{2\samplesize-2} \tilde{x}_{\{r,t\}}[\timeidx] \exp(- j 2 \pi k \timeidx/(2\samplesize\!-\!1)) \nonumber \\ 
& \stackrel{\eqref{equ_def_periodic_2N_x_mn}}{=} 
\!\sum_{\timeidx \in [\samplesize]} \big( \vx[\timeidx] \big)_{\{r,t\}}  \exp(- j 2 \pi k (\timeidx\!-\!1)/(2\samplesize\!-\!1)), \nonumber
\end{align} 
{for $k=0,\ldots,2\samplesize\!-\!2$.} 
Note that 
\begin{equation}
\label{equ_identity_element_DF}
\tilde{X}_{\{r,t\}}[k] = \big( \mathbf{D} \mathbf{F} \big)_{\{r,t\},k+1}.
\end{equation}  

Let us verify the equivalence of \eqref{equ_est_BT_factorization} and \eqref{equ_est_BT_sdm_theta_t} entry-wise. To this end{, for arbitrary but fixed $r,t \in [\coefflen]$,} consider the entry $\hat{s} \!\defeq\! \big( \ESDM(\theta) \big)_{r,t}$ of the SDM estimate given by \eqref{equ_est_BT_sdm_theta_t}. By inspecting \eqref{equ_est_BT_sdm_theta_t}, 
\begin{equation}
\label{equ_est_BT_sdm_theta_equiv_formula1}
\hat{s} = (1/\samplesize) \hspace*{-2mm} \sum_{\lagvar=-\samplesize\!+\!1}^{\samplesize\!-\!1}  \hspace*{-2mm} w[\lagvar] \exp(-2\pi \lagvar \theta) \cdot (\tilde{x}_{r} \otimes \tilde{x}_{t})[\lagvar], 
\end{equation} 
where $(\tilde{x}_{r} \!\otimes\! \tilde{x}_{t})[\lagvar] \!=\! \sum_{\timeidx=0}^{2 \samplesize \!-\!2} \tilde{x}_{r}[\timeidx+\lagvar] \tilde{x}_{t}[\timeidx]$ denotes the periodic autocorrelation function of $\tilde{x}_{r}[\timeidx]$ and $\tilde{x}_{t}[\timeidx]$. 
The DFTs $W[k]$ and $V[k]$ of the $(2\samplesize\!-\!1)$-periodic signals $w[\lagvar]\exp(-2\pi \lagvar \theta)$ and $(\tilde{x}_{r} \otimes \tilde{x}_{t})[\lagvar]$ are given by \cite[Ch. 8]{OppenheimSchaferBuck1998}, using $\theta_{k} \!\defeq\!2\pi(k\!-\!1)/(2\samplesize\!-\!1)$, 
\begin{equation} 
\label{equ_expr_DFT_W_V} 
W[k] = W(\theta + \theta_{k+1}) \mbox{ and } V[k] =  \tilde{X}_{r}[k]\tilde{X}^{*}_{t}[k],
\end{equation} 
respectively.
Using again \cite[Ch. 8]{OppenheimSchaferBuck1998}, we obtain from \eqref{equ_est_BT_sdm_theta_equiv_formula1} that
\begin{align}
\hat{s} 
 &  = \frac{1}{\samplesize(2\samplesize\!-\!1)} \sum_{k=0}^{2\samplesize\!-\!2} W[k]V^{*}[k] \nonumber \\[3mm]
 &  \stackrel{\eqref{equ_expr_DFT_W_V}}{=}\frac{1}{\samplesize(2\samplesize\!-\!1)}\sum_{k\in [2\samplesize\!-\!1]} W(\theta + \theta_{k}) \tilde{X}^{*}_{r}[k]\tilde{X}_{t}[k] \nonumber \\[3mm]
& \stackrel{\eqref{equ_identity_element_DF},\eqref{equ_def_diagonal_matrix_weight_function}}{=} \frac{1}{\samplesize} \sum_{k\in [2\samplesize\!-\!1]} (\mathbf{D}\mathbf{F})_{t,k} \big(\mathbf{W}(\theta)\big)_{k,k} \big( (\mathbf{D}\mathbf{F})^{H}\big)_{k,r}.
\end{align}  
Note that the last expression is nothing but the $(r,t)$-th entry of the RHS in \eqref{equ_est_BT_factorization}.

\section{Proof of Theorem \ref{them_support_selection_consistency_multitask_LASSO}}
\label{app_proof_them_support_selection_consistency_multitask_LASSO}

We will need the following lemma, which is a straightforward generalization of \cite[Thm.~6.1]{BuhlGeerBook}. 
\begin{lemma}
\label{thm_charact_error_multitask_LASSO}
Consider the multitask learning problem \eqref{equ_def_signal_model_given_cov_matrix} with parameter vector $\vbe(\cdot) \in \ell_{q}([0,1))$, observation vector $\mathbf{y}(\theta)$ and system matrix $\mX(\theta)$ defined by
\eqref{equ_def_parameter_vector_given_cov_matrix}, \eqref{equ_explicit_extression_y_DFT_tilde} and \eqref{equ_explicit_extression_X_DFT_tilde}, respectively.
Suppose,  
\begin{align}
\sup_{\theta \in [0,1)} \norm[\infty]{\herm{\vep}(\theta) \mX(\theta)  } < \frac{\lambda}{4} \mbox{, and } \gsupp({\bm \beta}(\cdot)) \subseteq \S, 
\label{eq:noiscorrbound}
\end{align}
with an index set $\S \subseteq [q]$ of size $\sparsity \! = \! |\S|$. 
If the system matrix $\mathbf{X}(\theta)$ possesses a positive multitask compatibility constant $\phi(\S)>0$, the mLASSO estimate $\hat{\vbe}[\mathbf{y}(\cdot),\mathbf{X}(\cdot)]$ given by \eqref{equ_def_multitask_LASSO} satisfies 
\vspace*{-2mm}
\begin{align}
\label{eq:condcharacterrormultitask}
\norm[1]{\vbe(\cdot) - \hat{ \vbe}(\cdot)} <  \frac{4 \lambda \sparsity}{\phi^{2}(\S)}. 
\end{align}
\end{lemma} 

Evaluating Lemma \ref{thm_charact_error_multitask_LASSO} for the specific choice $\lambda = \frac{\phi^{2}(\S) \beta_{\text{min}}}{8 \sparsity}$, we have that, under condition \eqref{equ_condition_corr_error_corr_support_selection1} (which ensures \eqref{eq:noiscorrbound}), the mLASSO estimate $\hat{\vbe}[\mathbf{y}(\cdot),\mathbf{X}(\cdot)]$ satisfies 
\begin{align}
\label{equ_proof_mLASSO_error_2_1_norm_error_lower_beta_min_2}
\norm[1]{\vbe(\cdot) - \hat{ \vbe}(\cdot)} < \beta_{\text{min}}/2. 
\end{align}
This implies, in turn, for any $r \in \gsupp(\vbe(\cdot))$, 
\begin{equation}
\norm[L^2] {\hat{\beta}_{r}(\cdot)} \!\geq\! \norm[L^{2}] {\beta_{r}(\cdot)} - | \norm[L^2] {{\beta}_{r}(\cdot)} -\norm[L^2] {\hat{\beta}_{r}(\cdot)} |  
\!\stackrel{\eqref{equ_beta_min_multitask_problem},\eqref{equ_proof_mLASSO_error_2_1_norm_error_lower_beta_min_2}}{>}\! \beta_{\text{min}}/2 \nonumber
\end{equation}
and similarly for any $r \in [\coefflen] \setminus \gsupp(\vbe(\cdot))$, 
\begin{equation}
\norm[L^2] {\hat{\beta}_{r}(\cdot)} \leq \norm[L^{2}] {\beta_{r}(\cdot)} + | \norm[L^2] {{\beta}_{r}(\cdot)} - \norm[L^2] {\hat{\beta}_{r}(\cdot)} |  \stackrel{\eqref{equ_proof_mLASSO_error_2_1_norm_error_lower_beta_min_2}}{<}  \beta_{\text{min}}/2. \nonumber
\end{equation} 
Thus, the set $\{ r: \norm[L^2] {\hat{\beta}_{r}(\cdot)} \geq\beta_{\text{min}}/2 \}$ coincides with the true generalized support $\gsupp(\vbe(\cdot))$.

\section{Proof of Lemma \ref{lem_relation_corr_hat_SDM_err}} 
\label{app_proof_lem_relation_corr_hat_SDM_err} 

{Let us recall the partitioning \eqref{equ_convenient_partitioning_SDM} of the SDM: 
\begin{equation}
\begin{pmatrix} \gamma(\theta) & \herm{\mathbf{c}}(\theta) \\[2mm]  \mathbf{c}(\theta) & \mathbf{G}(\theta) \end{pmatrix} \! \defeq \! \SDM(\theta). \nonumber
\vspace*{-1mm}
\end{equation} 
Analogously, we partition the SDM estimate $\ESDM(\theta)$} given by \eqref{equ_est_BT_sdm_theta_t} as 
\begin{equation}
\label{equ_convenient_partitioning_ESDM}
\begin{pmatrix} \hat{\gamma}(\theta) & \herm{\hat{\mathbf{c}}}(\theta) \\[2mm]  \hat{\mathbf{c}}(\theta) & \widehat{\mathbf{G}}(\theta) \end{pmatrix} \! \defeq \! \ESDM(\theta).
\end{equation} 
{For the sake of light notation, we consider throughout the remainder of this proof an arbitrary but fixed frequency $\theta$ and drop the argument of the 
frequency dependent variables, e.g., $\SDM(\theta)$, $\mathbf{G}(\theta)$, $\mathbf{c}(\theta)$, $\ESDM(\theta)$, $\widehat{\mathbf{G}}(\theta)$, $\hat{\mathbf{c}}(\theta)$ and so on.}

If we define the matrix $\mathbf{J} \in \mathbb{R}^{(\coefflen\!-\!1)\times \coefflen}$ by setting $J_{k,l}\! =\! 1$ if $l\!=\!k+1$ and $J_{k,l}\!=\!0$ else, we have 
\begin{equation}
\label{equ_relation_c_J_SDM_e_1}
\mathbf{c}  = \mathbf{J} \SDM \mathbf{e}_{1}. 
\end{equation} 

{
Consider the system matrix $\mX$ given by \eqref{equ_explicit_extression_X_DFT_tilde} and note that, 
by comparing \eqref{equ_grammian_of_mult_learning_problem_ESDM} with \eqref{equ_convenient_partitioning_ESDM},} we have 
\begin{equation}
\label{equ_grammian_X_widehat_G}
\mX^{H} \mX = \widehat{\mG}. 
\end{equation}
In what follows, 
we denote the $r$th columns of $\mX$, $\mG$ and $\widehat{\mG}$ by $\vx_{r}$, $\mathbf{g}_{r}$ and $\hat{\mathbf{g}}_{r}$, respectively.  

We also require a helpful identity for certain sub-matrices of the SDM:
\vspace*{-2mm}
\begin{equation}
\label{equ_identity_SDM_r_1}
\big( \SDM \big)_{r\!+\!1,1} = \mathbf{g}_{r}^{H} \mathbf{G}^{-1} \mathbf{c}.
\vspace*{-1mm}
\end{equation} 
This can be verified by
\vspace*{-2mm}
\begin{align}
 \mathbf{g}_{r}^{H} \mathbf{G}^{-1}  \mathbf{c} 
  & =  \mathbf{e}_{r}^{H} \mG \mathbf{G}^{-1}  \mathbf{c}   \nonumber \\[3mm]
 & \stackrel{\eqref{equ_relation_c_J_SDM_e_1}}{=} \mathbf{e}_{r}^{H} \mG \mG^{-1} \mathbf{J} \SDM \mathbf{e}_{1}  \nonumber \\[3mm]
 & =  \mathbf{e}_{r}^{H} \mathbf{J} \SDM \mathbf{e}_{1} \nonumber \\[3mm]
 & =  \big( \SDM \big)_{r\!+\!1,1}. \nonumber
 \vspace*{-1mm}
\end{align}
Note that 
\vspace*{-2mm}
\begin{align}
\label{equ_bound_corr_mag_1}
 |\herm{\mathbf{x}}_{r} \vep|  & \stackrel{\eqref{equ_def_signal_model_given_cov_matrix}}{=} |\herm{\mathbf{x}}_{r}(\mathbf{y} - \mathbf{X} \vbe )|  \\[3mm]
 & \stackrel{\eqref{equ_grammian_of_mult_learning_problem_ESDM},\eqref{equ_def_parameter_vector_given_cov_matrix}}{=}
 \big|  \big( \ESDM\big)_{r\!+\!1,1} \hspace*{-2mm}- \! \big( \hat{\mathbf{g}}_{r}\!-\!\mathbf{g}_{r} \big)^{H}  \mathbf{G}^{-1}\mathbf{c} 
\! - \! \mathbf{g}_{r}^{H}  \mathbf{G}^{-1} \mathbf{c} \big|. \nonumber
\end{align} 
Combining \eqref{equ_bound_corr_mag_1} with \eqref{equ_identity_SDM_r_1}, 
\vspace*{-1mm}
\begin{align}
\label{equ_bound_corr_mag_112}
 |\herm{\mathbf{x}}_{r} \vep | & \! =\!  \big|   \big( \ESDM  \! -Ê\! \SDM \big)_{r\!+\!1,1} \! - \! \big( \hat{\mathbf{g}}_{r} \!-\!\mathbf{g}_{r}\big)^{H}  \mathbf{G}^{-1} \mathbf{c} \big| \nonumber 
 \\[1mm]
 &  \stackrel{\eqref{equ_def_parameter_vector_given_cov_matrix}}{\leq}  \big|   \big( \ESDM  \! -Ê\! \SDM \big)_{r\!+\!1,1}\big| \! + \! \big| \big( \hat{\mathbf{g}}_{r}^{H} \!-\!\mathbf{g}_{r}^{H} \big) \vbe \big|.
\end{align} 
Applying the Cauchy-Schwarz inequality to the second term in \eqref{equ_bound_corr_mag_112} and using 
\vspace*{-2mm}
\begin{equation} 
| \gsupp(\vbe(\cdot)) | \stackrel{\eqref{eq:suppeqbeta}}{=} |\mathcal{N}(1)| \stackrel{\eqref{equ_def_maximum_node_degree}}{\leq} \maxdegree, 
\end{equation} 
we obtain 
\vspace*{-2mm}
\begin{align}
\label{equ_bound_corr_mag_1123}
 |\herm{\mathbf{x}}_{r} \vep| \! \leqÊ\! \big \|  \SDM \!-\!  \ESDM  \big\|_{\infty}(1 \! +\! \sqrt{s_{\text{max}}} \| \vbe \|_{2}).
 \vspace*{-2mm}
\end{align}
Inserting the bound 
\vspace*{-2mm}
\begin{align}
\| {\bm \beta} \|_{2} & \stackrel{\eqref{equ_def_parameter_vector_given_cov_matrix}}{=} \big\| \mathbf{G}^{-1} \mathbf{c} \big\|_{2} 
 \stackrel{\eqref{equ_uniform_bound_eigvals_specdensmatrix}}{\leq} U,  \nonumber
\end{align}
into \eqref{equ_bound_corr_mag_1123}, finally yields
\vspace*{-1mm}
\begin{align} 
 |\herm{\mathbf{x}}_{r} \vep| & \!\leq\!  \big \|  \SDM \!-\!  \ESDM  \big\|_{\infty} ( 1 + \sqrt{s_{\text{max}}} U ) \nonumber \\[2mm]
 & \!\leq\! 2  \big \|  \SDM \!-\!  \ESDM  \big\|_{\infty} \sqrt{s_{\text{max}}} U.\label{equ_upper_bound_mag_correlation}
 \vspace{-3mm}
\end{align}

\section{Proof of Lemma \ref{lem_compat_hat_multitask_problem_satisfied}} 
\label{app_proof_lem_compat_hat_multitask_problem_satisfied}
We first state an inequality which applies to any vector function ${\bm \beta}(\cdot) \in \ell_{q}([0,1))$ for some $q$.
In particular, 
\begin{align}
\label{equ_bound_mixed_2_1_norm_with_ell1_norm}
\int_{\theta=0}^{1} \| {\bm \beta}(\theta) \|^{2}_{1} d \theta & = \int_{\theta=0}^{1} \sum_{r \in [q]} | \beta_{r}(\theta)|  \sum_{r' \in [q]} | \beta_{r'}(\theta)|  d \theta \nonumber \\[1mm] 
& \stackrel{(a)}{\leq}   \sum_{r \in [q]} \sum_{r' \in [q]} \|\beta_{r}(\cdot)\|_{L^2}\|\beta_{r'}(\cdot)\|_{L^2} \nonumber \\[1mm]
& =  \| {\bm \beta}(\cdot) \|^{2}_{1},
\end{align} 
where step $(a)$ is due to the Cauchy-Schwarz inequality.
This, in turn, implies for any ${\bm \beta}'(\cdot) \in \mathbb{A}(\mathcal{S})$ (cf. \eqref{equ_norm_inequality_2_1_norm_compatibility}) that 
\begin{equation}
\label{equ_inequality_beta_prime_ell1_norm_ell_2_1_S}
\int_{\theta=0}^{1} \| {\bm \beta}'(\theta) \|^{2}_{1} d \theta \stackrel{\eqref{equ_bound_mixed_2_1_norm_with_ell1_norm}}{\leq} \| {\bm \beta}'(\cdot) \|^{2}_{1}
 \stackrel{\eqref{equ_norm_inequality_2_1_norm_compatibility}}{\leq} 16  \| {\bm \beta}'_{\mathcal{S}}(\cdot) \|^{2}_{1}.
\end{equation} 

Observe that   
\vspace*{-2mm}
\begin{align}
\label{equ_identy_proof_widehat_mX_vbe}
\norm[2]{ \mX(\cdot) \vbe(\cdot)}^2 
&= \int_{\theta=0}^{1} \herm{\vbe}(\theta)  \herm{\mX}(\theta) \mX(\theta) \vbe(\theta) d \theta \nonumber \\[1mm]
& \hspace*{-20mm} \stackrel{\eqref{equ_grammian_X_widehat_G}}{=} \! \int_{\theta=0}^{1} \hspace*{-5mm} \herm{\vbe}(\theta) \mathbf{G}(\theta) \vbe(\theta) d\theta \!+ \!\herm{\vbe}(\theta) \big[ \widehat{\mathbf{G}}(\theta) \!  -\! \mathbf{G}(\theta) \big] \vbe(\theta) d \theta.
\vspace*{-1mm}
\end{align} 
Since $\herm{\mathbf{a} }\mathbf{M} \mathbf{a} \leq \norm[\infty]{\mathbf{M}} \norm[1]{\mathbf{a}}^{2}$ for any vector $\mathbf{a} \in \mathbb{C}^{q}$ and matrix $\mathbf{M} \in \mathbb{C}^{q \times q}$, 
we obtain further 
\begin{align}
\label{equ_lower_bound_quadratic_form_13455}
& \norm[2]{ \mX(\cdot) \vbe(\cdot)}^2 
 \stackrel{\eqref{equ_identy_proof_widehat_mX_vbe}}{\geq} \nonumber \\[3mm]
& \hspace*{-2mm} \int_{\theta=0}^{1} \hspace*{-4mm} \herm{\vbe}(\theta) \mathbf{G}(\theta) \vbe(\theta) \! - \! \norm[\infty]{\widehat{\mathbf{G}}(\theta) \! -\! \mathbf{G}(\theta)} \norm[1]{\vbe(\theta)}^{2} d \theta \geq \nonumber \\[1mm]
& \hspace*{-2mm} \int_{\theta=0}^{1} \hspace*{-5mm} \herm{\vbe}(\theta) \mathbf{G}(\theta) \vbe(\theta) d\theta \!-\hspace*{-2mm}\sup_{\theta \in [0,1)}\hspace*{-2mm}\norm[\infty]{\ESDM(\theta) \! -\! \SDM(\theta)} \hspace*{-2mm} \int_{\theta=0}^{1}\hspace*{-3mm} \norm[1]{\vbe(\theta)}^{2} d \theta \stackrel{\eqref{equ_condition_max_diff_infty_norm_SDM_task}}{\geq}\nonumber \\[1mm]
& \hspace*{-2mm} \int_{\theta=0}^{1} \herm{\vbe}(\theta) \mathbf{G}(\theta) \vbe(\theta) d\theta  -   \frac{1}{32 s_{\text{max}}}  \norm[1]{\vbe(\cdot)}^{2} .
\end{align} 
Combining \eqref{equ_lower_bound_quadratic_form_13455} with \eqref{equ_inequality_beta_prime_ell1_norm_ell_2_1_S}, we have for any $\vbe(\cdot) \in \mathbb{A}(\mathcal{S})$, 
\begin{align}
s_{\text{max}} \frac{\norm[2]{ \mX(\cdot) \vbe(\cdot)}^2}{\| \vbe_{\mathcal{S}}(\cdot) \|^{2}_{1}}  & \! \geq \! s_{\text{max}} \frac{\int_{\theta=0}^{1} \herm{\vbe}(\theta) \mathbf{G}(\theta) \vbe(\theta) d\theta}{\| \vbe_{\mathcal{S}}(\cdot) \|^{2}_{1}} \! - \!  1/2Ê\nonumber \\[1mm]
& \stackrel{\eqref{equ_uniform_bound_eigvals_specdensmatrix}}{\geq}  1 -  1/2 =  1/2.
\end{align}

\section{Proof of Lemma \ref{lem_large_deviations_SDM_est}}
\label{app_proof_lem_large_deviations_SDM_est}

We will establish Lemma \ref{lem_large_deviations_SDM_est} by bounding $\big|\big(\ESDM(\theta) \! -\! \SDM(\theta)\big)_{k,l}\big|$ for a fixed pair $k,l \in [\coefflen]$ and then appealing to a union bound over all pairs $k,l \in [\coefflen]$. 

Set $\hat \sigma(\theta) \!\defeq\![\ESDM(\theta)]_{k,l}$, $\bar \sigma(\theta) \!\defeq\! \expect \{ [\ESDM(\theta)]_{k,l} \}$, $\sigma(\theta)\!\defeq\![\SDM(\theta)]_{k,l}$ and the bias $b(\theta)\!\defeq\! \sigma(\theta) \!-\!\expect\{ \hat{\sigma}(\theta) \}$. 
By the triangle inequality,
\begin{align}
& \PR\{\sup_{\theta \in [0,1)} \! |\hat \sigma \!-\! \sigma | \!\geq\! \nu \! +\! \mu_{x}^{(h_{1})} \}  
 \leq \nonumber \\[1mm] 
&  \PR\{ \sup_{\theta \in [0,1)} \! |\hat \sigma(\theta) \!-\! \bar \sigma(\theta) | \!+\! \sup_{\theta \in [0,1)}\! |b(\theta)| \!\geq\! \nu \! +\! \mu_{x}^{(h_{1})} \} \leq \nonumber \\[1mm]
& \PR\{ \sup_{\theta \in [0,1)} |\hat \sigma(\theta) \! -\! \bar \sigma(\theta) | \! \geq\! \nu \},
\label{eq:prhatsms}
\vspace*{-2mm}
\end{align}
where the last inequality holds since, for any $\theta \in [0,1)$, the bias satisfies $|b(\theta)| \leq \mu_{x}^{(h_{1})}$, which is verified next. 

With $\widetilde{\mathcal{N}} \defeq \{-\samplesize\!+\!1,\ldots,\samplesize\!-\!1\}$ and 
\vspace*{-1mm}
\begin{align}
\label{equ_expression_expect_ESDM}
\EX\{\ESDM(\theta)\} & \!\stackrel{\eqref{equ_est_BT_sdm_theta_t}}{=}\! \EX\bigg\{\frac{1}{N} \!\sum_{\lagvar=0}^{\samplesize-1} w[\lagvar] \hspace*{-4mm}\sum_{n \in [\samplesize\!-\!|\lagvar|]} \hspace*{-4mm}\vx[n+\lagvar] \transp{\vx}[\timeidx]   e^{-j 2 \pi \theta \lagvar } 
  \nonumber \\
  & \hspace*{-10mm} + \frac{1}{N} \! \sum_{\lagvar=-N+1}^{-1} \! w[\lagvar] \! \sum_{n\in [N-|\lagvar|]} \!\vx[\timeidx] \transp{\vx}[n-\lagvar]   e^{-j 2 \pi \theta \lagvar } \bigg\} \nonumber \\[3mm]
&=  \sum_{\lagvar \in \widetilde{\mathcal{N}} } w[\lagvar](1-|\lagvar|/N)  \ACF[\lagvar] e^{-j 2 \pi \theta \lagvar} \nonumber \\[3mm]
& \stackrel{\eqref{equ_finite_support_window_w}}{=}  \sum_{\lagvar \in \mathbb{Z}} w[\lagvar](1-|\lagvar|/N)  \ACF[\lagvar] e^{-j 2 \pi \theta \lagvar}, \nonumber
\vspace*{-1mm}
\end{align}
we obtain
\vspace*{-1mm}
\begin{align}
|\sigma(\theta)\!-\!\bar \sigma(\theta)| 
& \stackrel{\eqref{equ_def_weight_func_lem_large_deviations_SDM_est}}{=}  |     \sum_{\lagvar \in \mathbb{Z}}h_{1}[\lagvar] \big[ \ACF[\lagvar] \big]_{k,l} e^{-j 2 \pi \theta \lagvar}| 
& \stackrel{\eqref{equ_def_generic_moments_ACF}}{\leq} \mu_{x}^{(h_{1})}.  \\[-8mm]
\nonumber
\end{align}

Similarly,
\vspace*{-1mm}
\begin{align}
\label{equ_expression_error_sigma_k_l_linear}
\hat{\sigma} (\theta) \!-\!  \bar{\sigma}(\theta)  & \! \stackrel{\eqref{equ_est_BT_sdm_theta_t}}{=}\!
(1/\samplesize) \!\sum_{\lagvar \in \widetilde{\mathcal{N}}} \hspace*{-2mm}w[\lagvar]  q_{k,l}[\lagvar]  e^{-j 2 \pi \theta \lagvar},   \\[-8mm]
\vspace*{-4mm} \nonumber
\nonumber
\end{align}
where $q_{k,l}[\lagvar] \! \defeq \!  \mathbf{x}_{k}^{T} \mathbf{J}_{\lagvar} \vx_{l}\!-\!\expect \{  \mathbf{x}_{k}^{T} \mathbf{J}_{\lagvar} \vx_{l}\}$. Here, $\vx_k \! \defeq \! \transp{(x_k[1],\ldots,x_k[\samplesize])} \in \mathbb{R}^{\samplesize}$, $\vx_l \! \defeq \! \transp{(x_l[1],\ldots,x_l[\samplesize])} \in \mathbb{R}^{\samplesize}$ and the matrix $\mathbf{J}_{\lagvar} \in \{0,1\}^{\samplesize \times \samplesize}$ is defined element-wise as $\big( \mathbf{J}_{\lagvar} \big)_{v,w}\!=\!1$ for $w\!-\!vÊ\!=\! \lagvar$ and $\big( \mathbf{J}_{\lagvar} \big)_{v,w}\!=\!0$ else. Note that $\mathbf{J}_{\lagvar} = \mathbf{J}_{-\lagvar}^{T}$ and $\| \mathbf{J}_{\lagvar} \|_{2} \leq 1$.
By \eqref{equ_expression_error_sigma_k_l_linear}, for any $\theta \in [0,1)$, 
\begin{align}
\label{equ_upper_bound_excess_error_ell1norm_w_eta}
|\hat{\sigma}(\theta) \!-\! \expect \{ \hat{\sigma}(\theta) \} | & \!\leq\!  (1 / \samplesize) \! \sum_{\lagvar \in \widetilde{\mathcal{N}}} \! w[\lagvar] |q_{k,l}[\lagvar] |. \\[-8mm]
\nonumber
\end{align} 

In order to upper bounding the probability $\PR\{ \sup_{\theta \in [0,1)} |\hat \sigma(\theta) \! -\! \bar \sigma(\theta) | \! \geq\! \nu \}$, we now bound the probability of the event 
\begin{equation}
\label{equ_upper_bound_eta_single_correlations}
 (1/\samplesize) |q_{k,l}[\lagvar] |  \geq \tilde{\nu} 
\end{equation} 
by first considering the large deviation behavior of $(1/\samplesize) | q_{k,l}[\lagvar]  |$ for a specific $\lagvar$ and then using a union bound over all $\lagvar \in \widetilde{\mathcal{N}}$. 

Since we assume the process $\vx[\timeidx]$ to be Gaussian and stationary, the random vectors $\mathbf{x}_{k}$ and $\mathbf{x}_{l}$ in \eqref{equ_upper_bound_eta_single_correlations} are zero-mean normally distributed with Toeplitz covariance matrices $\mathbf{C}_{k} \!=\! \expect \{ \mathbf{x}_{k} \transp{\mathbf{x}_{k}} \}$ and 
$\mathbf{C}_{l} \!=\! \expect \{ \mathbf{x}_{l} \transp{\mathbf{x}_{l}} \}$, whose first row is given by  
$\big\{ \big(\mathbf{R}_{x}[\lagvar]\big)_{k,k} \big\}_{m \in [\samplesize]}$ and $\big\{ \big(\mathbf{R}_{x}[\lagvar]\big)_{l,l} \big\}_{m \in [\samplesize]}$, respectively. 
According to \cite[Lemma 4.1]{GrayToepliz}, and due to the 
Fourier relationship \eqref{equ_def_spectral_density_matrix}, 
we can bound 
the spectral norm of $\mathbf{C}_{k}$ as 
\vspace*{-2mm}
\begin{equation}
\| \mathbf{C}_{k} \|_{2} \leq \max_{\theta \in [0,1]} \big| \big(\SDM(\theta) \big)_{k,k} \big| \stackrel{(a)}{\leq} U. \nonumber
\vspace*{-1mm}
\end{equation} 
Here, step $(a)$ follows from \eqref{equ_uniform_bound_eigvals_specdensmatrix} together with the matrix norm inequality $\|\cdot\|_{\infty} \leq \|\cdot\|_{2}$ \cite[p.\ 314]{Horn85}. 
Similarly, one can also verify that $\| \mathbf{C}_{l} \|_{2} \leq U$. 
 
Therefore, for any $\tilde{\nu} < 1/2$, we can invoke Lemma \ref{lem_quadratic_form_real_valued_x_y} with the choices $\mathbf{x} \!=\! \mathbf{x}_{k}$, $\vy \!=\! \mathbf{x}_{l}$, $\lambda_{\text{max}} \!=\! U \!\geq\! 1$, $\mathbf{Q} \!=\! \mathbf{J}_{\lagvar}$ and $\lambda_{\text{max}}' \!=\! \| \mathbf{J}_{\lagvar} \|_{2} \leq 1$, yielding
\vspace*{-2mm}
\begin{align} 
\label{eq:conc}
\prob \{ (1/\samplesize) | q_{k,l}[\lagvar] |  \!\geq\!   \tilde{\nu} \}  
& \!\stackrel{\eqref{equ_lem_quadratic_form_real_valued_x_y}}{\leq}\! 2 \exp\bigg( - \frac{\samplesize \tilde{\nu}^{2}}{8 U^{2}} \bigg).  \\[-7mm]
\nonumber
\vspace*{-3mm}
\end{align} 
Then, by a union bound over all $\lagvar \in \widetilde{\mathcal{N}}$, 
\vspace*{-2mm}
\begin{equation}
\label{equ_upper_bound_prob_max_lagvar_quad_form_eta}
\prob\{ \max_{\lagvar \in \widetilde{\mathcal{N}}}\frac{1}{\samplesize} |  q_{k,l}[\lagvar] | \geq \tilde{\nu} \}\!\leq\!  2 \exp\bigg( \!-\! \frac{\samplesize \tilde{\nu}^{2}}{8 U^{2}} \!+\! \log(2\samplesize) \!\bigg), 
\vspace*{-1mm}
\end{equation}
and, in turn, 
\vspace*{-2mm}
\begin{align}
\label{equ_bound_sup_sigma_over_expect}
\hspace*{-2.5mm}\prob \{ \! \sup_{\theta \in [0,1)}\! | \hat{\sigma}(\theta)\!-\!\bar{\sigma}(\theta) | \! \geq\! \nu \} & \!\stackrel{\eqref{equ_upper_bound_excess_error_ell1norm_w_eta}}{\leq}\! \prob
\big\{ \max_{\lagvar \in \widetilde{\mathcal{N}}} \frac{1}{\samplesize} |q_{k,l}[\lagvar] | \!\geq\! \frac{\nu}{\|w[\cdot]\|_{1}} \big\} \nonumber \\[1mm] 
& \hspace*{-25mm} \stackrel{\eqref{equ_upper_bound_prob_max_lagvar_quad_form_eta}}{\leq}  2 \exp( - \samplesize \nu^{2} /(8  \|w[\cdot]\|^{2}_{1} U^{2}) + \log 2N) ).
\end{align} 
Applying \eqref{equ_bound_sup_sigma_over_expect} to \eqref{eq:prhatsms}, we have for any $\nu \leq 1/2$ that 
\vspace*{-2mm}
\[
\PR\{\sup_{\theta \in [0,1)} |\hat \sigma(\theta) \!-\! \sigma(\theta) | \geq \nu \!+\!  \mu_{x}^{(h_{1})} \} \! \leq\!  2 e^{-\frac{\samplesize \nu^2}{8  \| w [\cdot] \|_{1}^2 U^2}+ \log(2N)}.
\vspace*{-1mm}
\]
Another application of the union bound (over all $\coefflen^2$ pairs $(k,l) \in [\coefflen] \times [\coefflen]$) finally yields \eqref{eq:boundesterr}. 

\section{Large Deviations of a Gaussian Quadratic Form} 
\begin{lemma}
\label{lem_large_dev_gauss_quadratic_form_first_lemma}
Consider the quadratic form $q \! \defeq \! \mathbf{w}^{T} \mathbf{Q} \mathbf{w}$ with real-valued standard normal vector $\mathbf{w} \sim \mathcal{N}(\mathbf{0},\mathbf{I})$ 
and a real-valued symmetric matrix $\mathbf{Q} \in \mathbb{R}^{\samplesize \times \samplesize}$ with $\| \mathbf{Q} \|_{2} \leq \lambda_{\emph{max}}$. 
For any $\nu <1/2$, 
\vspace*{-1mm}
\begin{equation}
\prob\{ q - \expect\{q\}  \geq \samplesize \nu \} \leq  \exp \big( - \samplesize \nu^{2} / (8 \max\{\lambda^{2}_{\emph{max}},1\}) \big).
\vspace*{-1mm}
\end{equation}
\end{lemma}
\begin{proof} 
Our argument closely follows the techniques used in \cite{Bento2010}. 
In what follows, we will use the eigenvalue decomposition of $\mathbf{Q}$, i.e., 
\vspace*{-2mm}
\begin{equation} 
\label{equ_evd_q_proof_quadratic_form}
\mathbf{Q} = \sum_{l \in [\samplesize]} q_{l} \mathbf{v}_{l} \mathbf{v}_{l}^{T}, 
\vspace*{-2mm}
\end{equation} 
with eigenvalues $q_{l} \in \mathbb{R}$ and eigenvectors $\{ \mathbf{v}_{l} \}_{l \in [\samplesize]}$ forming an orthonormal basis for $\mathbb{R}^{\samplesize}$ \cite{golub96}. 
Note that, for any $l \in [\samplesize]$, we have $ |q_{l}| \! \leq \! \|\mathbf{Q} \|_{2} \! \leq \! \lambda_{\text{max}}$.
Based on \eqref{equ_evd_q_proof_quadratic_form}, we can rewrite the quadratic form $q= \mathbf{w}^{T} \mathbf{Q} \mathbf{w}$ as 
\vspace*{-1mm}
\begin{equation} 
\label{equ_quadratic_form_sum_iid_square_Gaussian}
q= \sum_{l \in [\samplesize]}  q_{l} z_{l}^{2},
\vspace*{-1mm}
\end{equation}
with i.i.d.\ standard Gaussian random variables $z_{l} \!\sim\! \mathcal{N}(0,1)$, for $l \in [\samplesize]$. 
We then obtain 
\begin{align} 
\label{equ_quadratic_form_tail_bound_1}
\prob \{ q - \expect\{q\}  \geq N \nu \} & = \prob \{ \mathbf{w}^{T} \mathbf{Q} \mathbf{w} - \expect\{\mathbf{w}^{T} \mathbf{Q} \mathbf{w}\} \geq N\nu \} \nonumber \\[1mm]
& \hspace*{-10mm} \stackrel{\eqref{equ_quadratic_form_sum_iid_square_Gaussian}}{=}   \prob\{ \sum_{l \in [\samplesize]}  q_{l} (z_{l}^{2} - 1) \geq N\nu \} \nonumber \\[1mm]
& \hspace*{-10mm}\stackrel{\gamma>0}{=} \prob\big\{ \gamma \big[ \sum_{l \in [\samplesize]}  q_{l} (z_{l}^{2} - 1) - N\nu \big] \geq 0  \big\} \nonumber \\[1mm]
& \hspace*{-10mm} \leq  \expect \big\{ \exp \big( \gamma \big[ \sum_{l \in [\samplesize]}  q_{l} (z_{l}^{2} - 1) - N\nu \big] \big)  \big\}, \\[-8mm]
\nonumber
\end{align}
for any positive $\gamma>0$.
In what follows, we set 
\begin{equation}
\label{equ_choice_for_gamma_eigth_lambda_max}
\gamma = \nu / (4 \max\{ \lambda_{\text{max}}^2,1 \} ), 
\end{equation} 
which implies, since $|q_{l}| < \lambda_{\text{max}}$ and $\nu < 1/2$ by assumption, 
\begin{equation}
\label{equ_how_small_we_choose_gamma}
2 |q_{l}| \gamma = 2 |q_{l}| \nu /(4\max\{ \lambda_{\text{max}}^2,1 \}) < 1/2. 
\end{equation}
Due to \eqref{equ_how_small_we_choose_gamma}, we also have $|\gamma q_{l}| < 1/2$ and can therefore use the identity 
\begin{equation}
\label{equ_identiy_exp_z_2} 
\expect\{ \exp( a  z_{l}^2) \}  = \frac{1}{\sqrt{1 - 2a}},
\end{equation}
valid for a standard Gaussian random variable $z_{l} \sim \mathcal{N}(0,1)$ and $|a| < 1/2$. 
Observe that
\begin{align}
\label{equ_bound_prob_deviation_mean_nu_first_step}
\prob \{ q \!-\! \expect\{q\}  \!\geq\! \samplesize\nu \} 
& \! \stackrel{\eqref{equ_quadratic_form_tail_bound_1}}{\leq}\!   \expect \big\{ \exp \big( \gamma \big[ \sum_{l \in [\samplesize]}  q_{l} (z_{l}^{2} \!-\! 1)\!-\! \samplesize \nu \big] \big)  \big\} \nonumber \\Ê
& \hspace*{-30mm} =  \exp \big(- \gamma \big[ \sum_{l \in [\samplesize]}  q_{l} \!+\! \samplesize \nu \big] \big)    \expect \big\{ \exp \big( \gamma  \sum_{l \in [\samplesize]}  q_{l} z_{l}^{2} \big)   \big\}. 
\end{align} 
Since the variables $z_{l}$ are i.i.d.,
\begin{align} 
\label{equ_bound_prob_deviation_mean_nu_second_step}
\expect \big\{ \! \exp \! \big( \gamma \! \sum_{l \in [\samplesize]} \! q_{l} z_{l}^{2} \big)   \big\} 
& \!\stackrel{\eqref{equ_identiy_exp_z_2}}{=}\!  \exp \bigg( \!\sum_{l \in [\samplesize]} \!-\!(1/2) \log(1\!-\!2\gamma q_{l})  \bigg). 
\end{align}
Inserting \eqref{equ_bound_prob_deviation_mean_nu_second_step} into \eqref{equ_bound_prob_deviation_mean_nu_first_step} yields
\vspace*{-1mm}
\begin{align}
\prob \{ q\!-\!\expect\{q\}\!\geq \! \samplesize \nu \}  & \leq \nonumber \\Ê
& \hspace*{-20mm}  \exp \big(-   \sum_{l \in [\samplesize]}  \big[ \gamma q_{l}   \!+\! \frac{1}{2}  \log(1\!-\!2\gamma q_{l})\big] \!-\! \gamma \samplesize \nu\big). \label{equ_bound_prob_deviation_mean_nu}
\vspace*{-1mm}
\end{align}
By \eqref{equ_how_small_we_choose_gamma}, we can then apply the inequality 
$\log(1-a) > - a - a^{2}$ (valid for $|a| < 1/2$) to \eqref{equ_bound_prob_deviation_mean_nu}, yielding  
\begin{align}
\label{equ_bound_prob_deviation_mean_nu_333}
\prob \{ q - \expect\{q\}  \geq \samplesize\nu \}   & \nonumber \\[1mm]Ê
& \hspace*{-30mm} \leq  \exp \big(\!   \sum_{l \in [\samplesize]} - \gamma q_{l} \!+\!  \gamma q_{l}\! +\!2 \gamma^{2} q_{l}^{2}  \! -\! \gamma \samplesize \nu  \big) \nonumber \\[1mm]Ê
& \hspace*{-30mm} \stackrel{|q_{l}| \leq \lambda_{\text{max}}}{\leq}  \exp \big( - \samplesize (\gamma \nu - 2\gamma^{2} \lambda^{2}_{\text{max}} ) \big). 
\end{align} 
Putting together the pieces, 
\begin{align}
 \prob \{ q - \expect\{q\}  \geq \samplesize \nu \} & \stackrel{\eqref{equ_bound_prob_deviation_mean_nu_333}}{\leq}  \exp \big( - \samplesize (\gamma \nu - 2\gamma^{2} \lambda^{2}_{\text{max}} ) \big) \nonumber \\[3mm]
 & \hspace*{-20mm} \stackrel{\eqref{equ_choice_for_gamma_eigth_lambda_max}}{\leq}  \exp \big( - \samplesize (\gamma \nu - (1/2) \gamma \nu \lambda^{2}_{\text{max}} / \max \{ \lambda^{2}_{\text{max}},1 \} ) \big) \nonumber \\[3mm]
 & \hspace*{-20mm} \leq \exp \big( - \samplesize \gamma \nu/2 \big)  \nonumber \\[3mm]
 & \hspace*{-20mm} \stackrel{\eqref{equ_choice_for_gamma_eigth_lambda_max}}{=} \exp \big( - \samplesize \nu^{2} / (8 \max \{ \lambda^{2}_{\text{max}},1 \}) \big).
\end{align}
\end{proof}

\begin{lemma}
\label{lem_quadratic_form_real_valued_x_y}
Consider two real-valued zero-mean random vectors $\mathbf{x} \in \mathbb{R}^{N}$ and $\vy \in \mathbb{R}^{N}$, such that the stacked vector $\mathbf{z} \defeq \big( \mathbf{x}^{T} \,\, \vy^{T}\big)^{T} \in \mathbb{R}^{2\samplesize}$ is zero-mean multivariate normally distributed, i.e., $\mathbf{z} \sim \mathcal{N}(\mathbf{0}, \mathbf{C}_{z})$ with covariance matrix $\mathbf{C}_{z} \defeq \expect\{ \mathbf{z} \mathbf{z}^{T} \}$. Let the individual covariance matrices satisfy $\| \mathbf{C}_{x} \|_{2} \leq \lambda_{\emph{max}}$, $\| \mathbf{C}_{y} \|_{2} \leq \lambda_{\emph{max}}$. We can then characterize the large deviations of the quadratic form $q\!\defeq\!\vy^{T} \mathbf{Q} \vx$, with an arbitrary (possibly non-symmetric) real-valued matrix $\mathbf{Q} \in \mathbb{R}^{\samplesize \times \samplesize}$ satisfying $ \| \mathbf{Q}\|_{2} \leq \lambda'_{\emph{max}}$, as  
\begin{equation} 
\label{equ_lem_quadratic_form_real_valued_x_y}
 \prob\{ | q - \expect\{ q\}|Ê\! \geq \! \samplesize \nu \} \!\leq\! 2\exp \big( - \samplesize \nu^{2} / (8 \max \{ \lambda'^{2}_{\emph{max}} \lambda^{2}_{\emph{max}},1\}) \big), 
\end{equation} 
valid for any $\nu < 1/2$. 
\end{lemma}

\begin{proof}

Introducing the shorthand $p(\nu) \defeq \prob\{ | q\!-\! \expect\{ q \}| \! \geq \! \samplesize \nu \}$, 
an application of the union bound yields
\begin{align} 
\label{equ_upper_bound_p_eta_union_bound}
p(\nu) 
    & \leq  \underbrace{\prob\{  q \!-\! \expect\{ q \} \geq \samplesize \nu \}}_{\defeq p_{+}(\nu)} +  \underbrace{\prob\{  q - \expect\{ q \} \leq - \samplesize \nu \}}_{\defeq p_{-}(\nu)}.
\end{align} 
We will derive an upper bound on $p(\nu)$ by separately upper bounding $p_{+}(\nu)$ and $p_{-}(\nu)$. The derivations are completely analogous and we will 
only detail the derivation of the upper bound on $p_{+}(\nu)$.

Defining the matrices $\mathbf{A}, \mathbf{B} \in \mathbb{R}^{N \times 2N}$ via the matrix square root of the covariance matrix $\mathbf{C}_{z}$, i.e., 
\begin{equation}
\label{equ_innov_matrices_A_B}
\begin{pmatrix} \mathbf{A} \\ \mathbf{B} \end{pmatrix} \defeq \mathbf{C}^{1/2}_{z},
\end{equation} 
{we have the following innovation representation for the random vectors} $\mathbf{x}$ and $\vy$:
\begin{equation} 
\label{equ_innov_repr_x_y_v}
\mathbf{x} = \mathbf{A} \mathbf{v}\mbox{, and } \vy=\mathbf{B} \mathbf{v},
\end{equation} 
where $\mathbf{v} \sim \mathcal{N}(\mathbf{0},\mathbf{I})$ is a standard normally distributed random vector of length $2\samplesize$. 
{Note that $\mathbf{C}_{x} = \mathbf{A} \mathbf{A}^{T}$ and $\mathbf{C}_{y} = \mathbf{B}\mathbf{B}^{T}$, which implies 
\begin{equation}
\label{equ_cond_spec_norm_A_B}
\| \mathbf{A} \|_{2} \!=\! \sqrt{ \| \mathbf{C}_{x} \|_{2}} \!\leq\! \sqrt{\lambda_{\text{max}}} \mbox{, and } \| \mathbf{B} \|_{2} \!=\! \sqrt{ \| \mathbf{C}_{y} \|_{2}} \!\leq\! \sqrt{\lambda_{\text{max}}}.
\end{equation}}

Let us further develop 
\begin{align}
\label{equ_derivation_p_plus_symmetrized}
 p_{+}(\nu) & = \prob\{  \vy^{T} \mathbf{Q} \mathbf{x} - \expect\{ \vy^{T} \mathbf{Q} \mathbf{x} \} \geq N\nu \} \nonumber \\[3mm]
 & \stackrel{\eqref{equ_innov_repr_x_y_v}}{=}  \prob\{  \mathbf{v}^{T} \mathbf{B}^{T} \mathbf{Q} \mathbf{A} \mathbf{v} - \expect\{ \mathbf{v}^{T} \mathbf{B}^{T} \mathbf{Q} \mathbf{A} \mathbf{v} \} \geq N \nu \} \nonumber \\[3mm]
  & \stackrel{(a)}{=}  \prob\{  \mathbf{v}^{T} \mathbf{M} \mathbf{v} - \expect\{ \mathbf{v}^{T} \mathbf{M} \mathbf{v} \} \geq N \nu  \},
\end{align} 
with the symmetric matrix 
\begin{equation} 
\label{equ_appendix_large_deviation_gaussian_quadratic_form_def_matrix_M}
\mathbf{M} = (1/2) [\mathbf{B}^{T} \mathbf{Q} \mathbf{A} + \mathbf{A}^{T} \mathbf{Q}^{T} \mathbf{B}].
\end{equation}  
In \eqref{equ_derivation_p_plus_symmetrized}, step $(a)$ follows from the identity $\mathbf{v}^{T} \mathbf{D} \mathbf{v} = (1/2)  [ \mathbf{v}^{T} \mathbf{D} \mathbf{v} + \mathbf{v}^{T} \mathbf{D}^{T} \mathbf{v}]$, which holds for an arbitrary matrix $\mathbf{D} \in \mathbb{R}^{2N \times 2N}$. 
Combining \eqref{equ_appendix_large_deviation_gaussian_quadratic_form_def_matrix_M} with \eqref{equ_cond_spec_norm_A_B} yields 
\begin{align}
\label{equ_upper_bound_spectral_norm_M}
\| \mathbf{M} \|_{2} & \stackrel{\eqref{equ_appendix_large_deviation_gaussian_quadratic_form_def_matrix_M}}{=} 
(1/2)  \| \mathbf{B}^{T} \mathbf{Q} \mathbf{A} + \mathbf{A}^{T} \mathbf{Q}^{T} \mathbf{B} \|_{2}  \nonumber \\[3mm] 
& \stackrel{(a)}{\leq} (1/2) ( \|\mathbf{B}^{T}\|_{2} \| \mathbf{Q} \|_{2} \| \mathbf{A}Ê\|_{2} +  \|\mathbf{A}^{T}\|_{2} \| \mathbf{Q}^{T} \|_{2} \| \mathbf{B}Ê\|_{2})  \nonumber \\[3mm]Ê
&  =  \|\mathbf{B}\|_{2} \| \mathbf{Q} \|_{2} \| \mathbf{A}Ê\|_{2} \nonumber \\[3mm]Ê 
& \stackrel{\eqref{equ_cond_spec_norm_A_B}} \leq  \lambda_{\text{max}} \lambda'_{\text{max}},
\end{align}
where step $(a)$ is due to the triangle inequality and submultiplicativity of the spectral norm. 
Using \eqref{equ_upper_bound_spectral_norm_M}, the application of Lemma \ref{lem_large_dev_gauss_quadratic_form_first_lemma} to \eqref{equ_derivation_p_plus_symmetrized} yields 
\vspace*{-1mm}
\begin{equation} 
\label{equ_upper_bound_probability_p_plus}
p_{+}(\nu)  \leq  \exp \big( - \samplesize \nu^{2} / (8 \max \{ \lambda'^{2}_{\text{max}} \lambda^{2}_{\text{max}},1\}) \big),
\vspace*{-1mm}
\end{equation} 
and, similarly, 
\vspace*{-1mm}
\begin{equation} 
\label{equ_upper_bound_probability_p_minus}
p_{-}(\nu)  \leq \exp \big( - \samplesize \nu^{2} / (8 \max \{ \lambda'^{2}_{\text{max}} \lambda^{2}_{\text{max}},1\}) \big).
\vspace*{-1mm}
\end{equation} 
Inserting \eqref{equ_upper_bound_probability_p_plus} and \eqref{equ_upper_bound_probability_p_minus} into \eqref{equ_upper_bound_p_eta_union_bound} finally yields 
\vspace*{-1mm}
\begin{equation}
p(\nu) \leq 2   \exp \big( - \samplesize \nu^{2} / (8 \max \{ \lambda'^{2}_{\text{max}} \lambda^{2}_{\text{max}},1\}) \big). \nonumber
\vspace*{-5mm}
\end{equation} 
\end{proof}

\bibliographystyle{IEEEtran}
\bibliography{/Users/ajung/Work/LitAJ_ITC.bib,/Users/ajung/Work/tf-zentral}

\end{document}